\documentclass[twoside]{article}

\usepackage[accepted]{aistats2020}

\usepackage[round]{natbib}

\usepackage[utf8]{inputenc} 
\usepackage[T1]{fontenc}    
\usepackage{hyperref}       
\usepackage{url}            
\usepackage{booktabs}       
\usepackage{longtable}
\usepackage{xtab,afterpage,multicol}
\usepackage{microtype}      
\usepackage{graphicx}
\usepackage{amsmath,amsthm,amssymb}
\usepackage[mathscr]{euscript}
\usepackage{dsfont}
\usepackage[ruled,linesnumbered]{algorithm2e}
\usepackage{subcaption}
\usepackage{makecell}
\usepackage[table]{xcolor}
\usepackage{mathtools}

\usepackage[final]{fixme}
\fxsetup{inline,nomargin,theme=color}
\usepackage[inline]{enumitem}

\usepackage{siunitx}
\usepackage{csquotes} 

\newcommand{\undersett}[2]{\underset{\text{#1}}{#2}}

\newtheorem{thmlem}{Lemma}
\newtheorem{thmcol}{Corollary}
\newtheorem{thmprop}{Proposition}
\newtheorem{thmthm}{Theorem}

\theoremstyle{definition}

\makeatletter
\newtheorem*{rep@theorem}{\rep@title}
\newcommand{\newreptheorem}[2]{%
\newenvironment{rep#1}[1]{%
 \def\rep@title{#2 \ref{##1} (Restated)}%
 \begin{rep@theorem}}%
 {\end{rep@theorem}}}
\makeatother

\newreptheorem{theorem}{Theorem}
\newreptheorem{lemma}{Lemma}
\newreptheorem{col}{Corollary}
\newreptheorem{def}{Definition}

\def\scrC{\mathscr{C}}

\def\bx{\mathbf{x}}

\def\bK{\mathbf{K}}

\def\cB{\mathcal B}
\def\hcB{\hat{\cB}}

\def\cD{\mathcal D}
\def\cI{\mathcal I}
\def\cJ{\mathcal J}
\def\cC{\mathcal C}
\def\cT{\mathcal T}
\def\cX{\mathcal X}
\def\cY{\mathcal Y}
\def\cF{\mathcal F}

\def\cH{\mathcal H}
\def\cU{\mathcal U}
\def\cK{\mathcal K}
\def\cS{\mathcal S}

\def\cF{\mathcal{F}}

\def\cO{\mathcal{O}}
\def\hB{\hat{B}}
\def\hcO{\hat{\mathcal{O}}}
\def\tcO{\tilde{\mathcal{O}}}
\def\hR{\hat{R}}

\def\hcS{\hat{\mathcal{S}}}

\def\tcB{\tilde{\mathcal{B}}}

\def\ho{\hat{o}}
\def\hs{\hat{s}}

\def\rulet{}

\def\supp{\textnormal{supp}}

\DeclareMathOperator*{\E}{\mathbb{E}}

\DeclareMathOperator*{\argmin}{arg\,min}

\begin{document}

\runningauthor{Michael Oberst*, Fredrik D. Johansson*, Dennis Wei* et al.}

\twocolumn[
\aistatstitle{Characterization of Overlap in Observational Studies}

\aistatsauthor{Michael Oberst* \And Fredrik D. Johansson* \And Dennis Wei* }%
\aistatsaddress{ MIT-IBM Watson AI Lab \\ MIT, CSAIL \& IMES \And Chalmers University of Technology \And MIT-IBM Watson AI Lab \\ IBM Research}%
\aistatsauthor{ Tian Gao \And Gabriel Brat \And David Sontag \And Kush R. Varshney }%
\aistatsaddress{ MIT-IBM Watson AI Lab \\ IBM Research \And Harvard Medical School \And MIT-IBM Watson AI Lab \\ MIT, CSAIL \& IMES \And MIT-IBM Watson AI Lab \\ IBM Research }%
]

\begin{abstract}
  Overlap between treatment groups is required for non-parametric estimation of causal effects.  If a subgroup of subjects always receives the same intervention, we cannot estimate the effect of intervention changes on that subgroup without further assumptions.  When overlap does not hold globally, characterizing local regions of overlap can inform the relevance of causal conclusions for new subjects, and can help guide additional data collection. To have impact, these descriptions must be interpretable for downstream users who are not machine learning experts, such as policy makers.  We formalize overlap estimation as a problem of finding minimum volume sets subject to coverage constraints and reduce this problem to binary classification with Boolean rule classifiers. We then generalize this method to estimate overlap in off-policy policy evaluation. In several real-world applications, we demonstrate that these rules have comparable accuracy to black-box estimators and provide intuitive and informative explanations that can inform policy making. 
\end{abstract}

\section{INTRODUCTION}
\label{sec:introduction}
To accurately estimate the causal effect of an intervention, it is essential that intervention alternatives have been observed in comparable contexts, i.e., that there is {\em overlap} between the distributions of individuals receiving each intervention \citep{rosenbaum1983central,d2017overlap}. In randomized experiments, overlap is guaranteed for the study population by randomizing the intervention. However, this is not the case in observational studies where interventions are chosen according to an existing, in some cases deterministic, policy. In such settings, overlap may hold only for an unidentified subset of cases, with the causal effect being unidentifiable outside of this subset. We motivate our paper with the following use cases:

\emph{Scenario 1: From study to policy.} When researchers publish the findings of a clinical trial, they also share the eligibility criteria (e.g., {\em Age $\geq$ 18, Serum M protein $\geq$ 1g/dl or Urine M protein $\geq$ 200 mg/24 hrs, Recent diagnosis} \citep{myelomatrial2}) and cohort statistics in order to characterize the cohort of study subjects.
This gives policy makers means to assess the external validity of the results, i.e., to whom the results apply. We seek to provide the same for observational studies, with our algorithms producing an interpretable description of subjects with treatment group overlap. 

\emph{Scenario 2: Evaluating guidelines.} There are over 471 different guidelines for how to manage hypertension \citep{guidelines}. We could evaluate these---and new guidelines---using off-policy evaluation methods~\citep{Precup2000} on observational data derived from electronic medical records. Off-policy evaluation of a guideline is only possible on subsets of the population where there is some probability that the guideline was followed (which we will also call overlap). The estimated policy value should be accompanied by a description of the validity (overlap) region.

Beyond causal estimation, overlap is of interest in many other branches of machine learning: In domain adaptation, the overlap between source and target domains is the set of inputs for which we can expect a trained model to transfer well~\citep{ben2010theory,johansson2019support}; In classification, overlap between inputs with different labels signifies regions that are hard to classify; In algorithmic fairness~\citep{dwork2012fairness}, overlap between protected groups may shed light on disparate treatment of individuals from different groups who are otherwise comparable in task-relevant characteristics; In reinforcement learning, lack of overlap has been identified as a failure mode for deep Q-learning using experience replay~\citep{fujimoto19a}.

\begin{figure}[t!]
    \centering
    \includegraphics[width=.8\columnwidth]{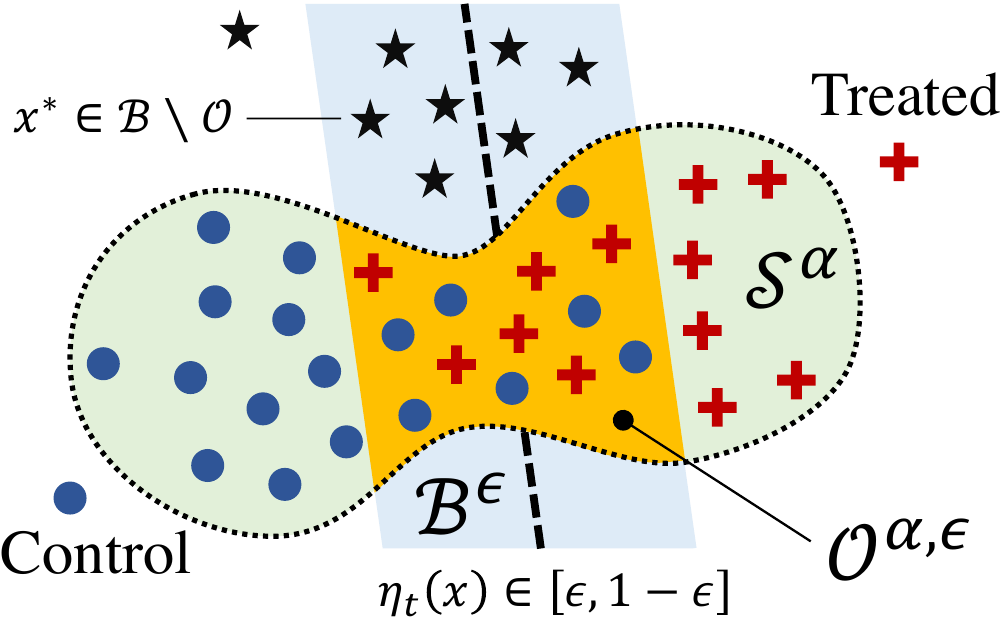}%
    \caption{Overlap $\cO^{\alpha, \epsilon}$ between treatment groups with joint support $\cS^\alpha$. A point $x^* $ has group propensity $\eta_t$ bounded away from 0 and 1, but is outside of $\cO^{\alpha, \epsilon}$.}
    \label{fig:overlap}
\end{figure}

Our main contributions are as follows:
\begin{enumerate*}[label=(\roman*)]
    \item We propose desiderata in overlap estimation, and note how existing methods fail to satisfy them.
    \item We give a method for \emph{interpretable characterization of distributional overlap}, which satisfies these desiderata, by reducing the problem to two binary classification problems, and using a linear programming relaxation of learning optimal Boolean rules.
    \item We give generalization bounds for rules minimizing empirical loss.
    \item We demonstrate that small rules often perform comparably to black-box estimators on a suite of real-world tasks.
    \item We evaluate the interpretability of rules for describing treatment group overlap in post-surgical opioid prescription in a user study with medical professionals.
    \item We show how a generalized definition and method applies to policy evaluation and apply it to describing overlap in policies for antibiotic prescription.
\end{enumerate*}

\section{RELATED WORK}
\label{sec:related}
Treatment group overlap is a central assumption in the estimation of causal effects from observational data. Comparing group-specific covariate bounds and lower-order moments is a common first step in assessing overlap~\citep{rosenbaum2010design,zubizarreta2012using,fogarty2016discrete} but fails to identify local regions of overlap when they exist (see the example of $\cO^{\alpha, \epsilon}$ in Figure~\ref{fig:overlap}). An alternative is to estimate the \emph{treatment propensity}---the probability that a subject was prescribed treatment. Treatment propensities bounded away from $0$ and $1$ at a point $X$ indicates that treatment groups overlap at $X$~\citep{rosenbaum1983central,li2018balancing}. 

In studies with partial overlap, it is common to restrict the study cohort by thresholding treatment propensity or discarding unmatched subjects after applying matching methods~\citep{rosenbaum1989optimal,iacus2012causal,kallus2016generalized,visconti2018handling}. For example, \citet{crump2009dealing} proposed an optimal propensity threshold that minimizes the variance of the estimated average treatment effect on a sub-population. However, neither propensity thresholding nor matching are sufficient for guiding policy in new cases: they do not provide a self-contained, interpretable description of where treatment groups overlap \emph{within} the study, nor do they provide insight into \emph{external} validity by describing the limits of the study cohort. 

\citet{fogarty2016discrete} address the first concern above by learning ``interpretable study populations''  through identifying the largest axis-aligned box that contains only subjects with bounded propensity. However, this approach is very limited in capacity and does not address external validity. For this reason, we strive to provide interpretable descriptions of overlap, both in terms of treatment propensity and the study support. 

Rule-based models have been considered in classification tasks~\citep{rivest1987learning,angelino2017learning,yang2017scalable,lakkaraju2016interpretable,wang2017bayesian,dash2018boolean,freitas2014comprehensible,wang2015falling}, subgroup discovery~\citep{herrera2011overview} and density estimation~\citep{ram2011density,goh2015cascaded} but have to the best of our knowledge not been applied or tailored to support or overlap estimation. 

\section{DEFINING OVERLAP}
\label{sec:background}
We address \emph{interpretable description of population overlap}. Our primary motivation is to aid \emph{policy making} based on observational studies, the success of which relies on understanding and communicating the studies' validity region---the set of cases for which there is evidence that a particular policy decision is preferable.
We identify the following desiderata for descriptions of overlap:
\begin{enumerate*}[label=(D.\arabic*)]
  \item They cover regions where all populations (treatment groups) are well-represented;\label{des:1}
  \item They exclude all other regions, including those outside the support of the study (see Figure~\ref{fig:overlap});\label{des:2}
  \item They can be expressed using a small set of simple rules.\label{des:3}
\end{enumerate*} 
Next, we define overlap according to \ref{des:1} and \ref{des:2}. 
We address \ref{des:3} in Section~\ref{sec:algorithm}.

Let subjects $i = 1, ..., m$ be observed through samples $(x_i, t_i)$ of covariates $X \in \cX \subseteq \mathbb{R}^d$ and a group indicator $T \in \cT$. In our running example, $X$ represents patient attributes and $T$ their treatment. We assume that subjects are independently and identically distributed according to a density $p(X, T)$, and that $\cX$ is bounded. Let $p_t(X) := p(X\mid T=t)$ denote the covariate density of group $t \in \cT$ and $\eta_t(x) := p(T=t \mid X=x)$ the propensity of membership in group $t \in \cT$ for subjects with covariates $x\in \cX$. We denote the probability mass of a set $S\subseteq \cX$ under $p$ by $P(S) := \int_{x \in S} dp$ and the support of $p$ by $\supp(p) := \{x\in \cX : p(x) > 0\}$.

In the common case of two groups, $\cT = \{0,1\}$, overlap is typically defined as either a) the intersection of supports, $\supp(p_0) \cap \supp(p_1)$, or b) the set of covariate values for which all group propensities $\eta_t$ are bounded away from zero~\citep{d2017overlap,li2018balancing}. We let  $\cB^\epsilon$ denote this latter set of values with \emph{$\epsilon$-bounded propensity} for a fixed parameter $\epsilon\in (0,1)$ and an arbitrary set of groups $\cT$,
\begin{equation}\label{eq:boundprop}
\cB^\epsilon := \{x \in \cX ; \forall t \in \cT \colon \eta_t(x) > \epsilon\}~.
\end{equation}
Neither $\cB^\epsilon$ nor the support intersection fully capture our desired notion of overlap: The former does not satisfy \ref{des:2} since a point may have bounded propensity (true or estimated) but lie outside the population support $\supp(p)$ (see Figure~\ref{fig:overlap}). Note that interpretable description alone does not address this. The latter is non-informative for variables with infinite support (e.g., a normal random variable), and even with finite support, we may wish to exclude distant outliers.

Our preferred definition of overlap combines the requirement of bounded propensity with a generalization of support called $\alpha$-\emph{minimum-volume sets}~\citep{scholkopf2001estimating}. Let $\scrC$ be a set of measurable subsets of $\cX$, let $V(C)$ denote the volume of a set $C \in \scrC$. An $\alpha$-\emph{minimum-volume} set $\cS^\alpha$ of $p$ is then
\begin{equation}\label{eq:MVset}
\cS^\alpha := \argmin_{C}\{V(C) \;; P(C) \geq \alpha, C \in \scrC\}~,
\end{equation}
with $\cS^1=\supp(p)$. For $\alpha < 1$, $\cS^\alpha$ is not always unique, but the intersection $S$ of two $\alpha$-MV sets has mass $P(S) \geq 2\alpha-1$. In this work, we let $\alpha < 1$ in order to handle distributions with infinite support and unwanted outliers, and \emph{refer to $\cS^\alpha$ as the support of $p$}. 
We define the \emph{$\alpha, \epsilon$-overlap set}, for $\alpha, \epsilon \in (0,1)$, to be 
\begin{equation}\label{eq:overlapdef}
\cO^{\alpha, \epsilon} :=  \cS^\alpha \cap \cB^\epsilon~.
\end{equation}
We define the problem of overlap estimation under definition \eqref{eq:overlapdef} as \emph{characterizing the set $\cO^{\alpha, \epsilon}$ given thresholds $\alpha$ and $\epsilon$}. In line with \ref{des:3}, these characterizations should be useful in policy making, and interpretable by domain experts, at small cost in accuracy. For notational convenience, we sometimes leave out superscripts from $\cS^\alpha, \cB^\epsilon$ and $\cO^{\alpha, \epsilon}$, assuming that $\alpha, \epsilon$ are fixed.

\textbf{Remark.} Defining overlap instead as the intersection of group-specific $\alpha$-MV sets is feasible, but scales poorly with $|\cT|$; it does not facilitate the generalization to policy evaluation described below; and the intersection of many descriptions may be hard to interpret.

\subsection{Generalization to Policy Evaluation}
\label{sec:overlaprulespolicy}
The definition of $\cB^\epsilon$ in \eqref{eq:boundprop} is motivated by causal effect estimation---comparison of outcomes under two or more alternative interventions.  We may instead be interested in policy evaluation, which involves estimating the expected outcome under a conditional intervention $\pi$, which assigns a treatment $t$ to each $x$ following a conditional distribution $\pi(T | X)$~\citep{Precup2000}.  To perform this evaluation, we only require that the propensity $p(T|X)$ of observed treatments be bounded away from zero for treatments which have non-zero probability under $\pi$. To describe the inputs for which this is satisfied, we generalize $\cB^\epsilon$ to be a function of the target policy $\pi$,
\begin{equation}\label{eq:boundproppolicy}
\cB^\epsilon(\pi) := \{x \in \cX ; \forall t : \pi(t\mid x) > 0 : \eta_t(x) > \epsilon\}~.
\end{equation}
More details are given in the supplement regarding the use of OverRule in this setting.

\section{OVERRULE: BOOLEAN RULES FOR OVERLAP}
\label{sec:algorithm}

We propose OverRule\footnote{Code available at \url{https://github.com/clinicalml/overlap-code}}, an algorithm for identifying the overlap region $\cO$ in \eqref{eq:overlapdef} by first estimating the $\alpha$-MV support set $\cS$~\eqref{eq:MVset} and then the bounded-propensity set $\cB$~\eqref{eq:boundprop} restricted to $\cS$, thereby satisfying desiderata~\ref{des:1}--\ref{des:2}.  We aim to fulfill desideratum~\ref{des:3} by using Boolean rules---logical formulae in either disjunctive (DNF) or conjunctive (CNF) normal form---which have received renewed attention because of their interpretability~\citep{dash2018boolean,su2016learning}. See Figures~\ref{fig:jobsrules}--\ref{fig:opioid_rules} for examples of learned rules. 
OverRule proceeds in the following steps: 
\begin{enumerate}[label=(\roman*)]
\item Fit $\alpha$-MV set $\hcS^\alpha$ of $p(X)$ using Boolean rules\label{step:srule}%
\item Fit model of group propensity $\hat{\eta}_{(\cdot)}$ over $\hcS^\alpha$ and let $\tilde{b}(x) = \prod_{t\in \cT}\mathds{1}[\hat{\eta}_t(x) > \epsilon]$ define membership in $\tcB^\epsilon$\label{step:bbase}%
\item Approximate $\tcB^\epsilon$ using Boolean rules to yield $\hcB^\epsilon$ \label{step:brule} and estimate overlap region by $\hcO^{\alpha, \epsilon} = \hcB^\epsilon \cap \hcS^\alpha$.
\end{enumerate}

In this section, we demonstrate how steps \ref{step:srule} \& \ref{step:brule} can be reduced to binary classification. This enables us to exploit the many existing methods for rule-based classification~\citep{freitas2014comprehensible} to improve the interpretability of $\hcO$. Finally, we give results bounding the generalization error of estimates of both $\cS$ and $\cS \cap \cB$.

\textbf{Remark.} It was observed in evaluations with a medical practitioner that fitting rules for $\cS$ and $\cB$ separately improved interpretability as it makes clear which rules apply to which task and prevents the bulk of the rules from being consumed by one of the two tasks.
%
%
\subsection{Estimation of \texorpdfstring{$S^\alpha$}{Salpha} as Binary Classification}
\label{sec:MVruleSet}
In the first step of OverRule, we learn a Boolean rule to approximate the $\alpha$-MV set $\cS^\alpha$ of the marginal distribution $p(X)$ by reducing the problem to binary classification between observed samples $\cD \coloneqq \{x_i\}_{i=1}^m$ and uniform background samples. 
For clarity, we focus only on DNF rules---disjunctions of conjunctive clauses such as $(\text{Age} < 30 \land \text{Female}) \lor (\text{Married})$. As pointed out by \cite{su2016learning}, a CNF rule can be learned by swapping class labels and fitting a DNF rule. 

We adapt previous notation and let $\scrC$ be a class of candidate $\alpha$-MV sets $\cC$ corresponding to Boolean rules, i.e., each $\cC$ consists of the points in $\cX$ that satisfy a rule. 
We will often not distinguish between a rule and its corresponding set $\cC$ and thus will speak of the ``volume'' of a rule or clause. 
We aim to solve a normalized and regularized version of the $\alpha$-MV problem in \eqref{eq:MVset},
\begin{equation}
\argmin_{\cC\in\scrC} \;
Q(\cC) := \undersett{Volume}{\bar{V}(\cC)} + \undersett{Regularization\;\;\;}{R(\cC) \ \ \text{ s.t. }}\undersett{Coverage}{P(\cC) \geq \alpha}
\label{eq:MVobj}
\end{equation}
where the volume $\bar{V}(\cC) = V(\cC) / V(\cX) \in [0,1]$ is normalized to that of $\cX$. We assume that the regularization term $R(\cC)$ controls complexity by placing penalties $\lambda_0$ on each clause in the rule and $\lambda_1$ on each condition in a clause. Thus, for a Boolean rule with clauses $k = 1,\dots,K$, each with $p_k$ conditions, we have\footnote{It is possible to generalize \eqref{eqn:R} to place different penalties on different conditions but we adopt \eqref{eqn:R} for simplicity.}
\begin{equation}\label{eqn:R}
R(\cC) = K\lambda_0 + \lambda_1 \sum_{k=1}^K p_k.
\end{equation}
It is also assumed that the trivial ``all-true'' and ``all-false'' rules have complexity $R(\cC) = 0$.

The volume $\bar{V}(\cC)$ may be difficult to compute repeatedly during optimization 
and $\scrC$ is often too large to allow pre-computation of $\bar{V}(\cC)$ for all $\cC$. In particular, for DNF rules, each $\cC$ is a union of potentially several overlapping clauses (see Figures~\ref{fig:jobsrules}--\ref{fig:opioid_rules} or the illustration in the supplement); even if the volume spanned by each clause is quick to compute on the fly, the overall volume may not be. 
As an alternative, the normalized volume $\bar{V}(\cC)$ can be estimated by means of uniform samples $\{x_{m+1},\dots,x_{m+n}\}$ over $\cX$. Let $\cU$ be the index set of these uniform samples. 
Then, $\frac{1}{n} \sum_{i\in\cU} \mathds{1}[x_i \in \cC]$ is distributed as a scaled binomial random variable with mean $\bar{V}(\cC)$ and variance $\bar{V}(\cC) (1-\bar{V}(\cC)) / n$. Theorem~\ref{thm:MV} below provides guidance in selecting the number of uniform samples $n$ to ensure a good estimate.

Given the above empirical estimator of volume, we reduce problem \eqref{eq:MVobj} to a classification problem between the marginal density $p(X)$ and a uniform distribution over $\cX$. This reduction was also mentioned in the conclusion of \cite{scott2006learning}. We also replace the probability mass constraint with its empirical version over $\cD$ with $\cI = \{1,\dots,m\}$. The result is a Neyman-Pearson-like classification problem with a false negative rate constraint of $1-\alpha$ (instead of the usual false positive constraint), as given below.
\begin{equation}\label{eqn:binaryclass}
\begin{aligned}
\hcS := \; & \underset{\mathcal{C}}{\argmin}
& & \frac{1}{\lvert\cU\rvert} \sum_{i\in\cU} \mathds{1}[x_i \in \cC]  + R(\cC) \\
& \text{subject to}
& & \sum_{i\in\cI} \mathds{1}[x_i \in \cC] \geq \alpha m~.
\end{aligned}
\end{equation}

The following theorem bounds the regret of the minimizer of \eqref{eqn:binaryclass} with respect to \eqref{eq:MVobj} and is proven in the supplement. The assumption of binary variables simplifies the analysis and is not a fundamental limitation.
\begin{thmthm}\label{thm:MV}
Let $q^*(\alpha)$ denote the minimum regularized volume attained in \eqref{eq:MVobj} over the class of DNF rules with probability mass $\alpha$. Assume that a) the regularization $R$ follows \eqref{eqn:R} with fixed parameters $\lambda_0, \lambda_1$, b) all variables $X_j$ are binary-valued, and c) the class $\scrC$ is restricted to rules satisfying necessary conditions of optimality for \eqref{eq:MVobj} (see Lemmas in the supplement). Then with probability greater than $1-2\delta$, the empirical estimate $\hat{\cS}$ in \eqref{eqn:binaryclass} satisfies
\[
Q(\hat{\cS}) \leq q^*(\alpha + \epsilon_m) + 2\epsilon_n \quad \text{and} \quad 
    P(\hat{\cS}) \geq \alpha - \epsilon_m,
\]
where
$\epsilon_m = \sqrt{\frac{\lambda_1^{-1} \log(2d) + \left\lfloor 1 + \log_2 \lambda_1^{-1} \right\rfloor \log \lambda_1^{-1} + \log(4/\delta)}{2m}}$
and $\epsilon_n$ is defined analogously.
\end{thmthm}

\textbf{Remark.} The error term $\epsilon_m$ bounds the amount by which the probability constraint may be violated and contributes $q^*(\alpha+\epsilon_m) - q^*(\alpha)$ to the possible regret. Given the number of data samples $m$, penalty $\lambda_1$ ($\lambda_0$ does not appear in this simplified bound) could be chosen to keep $\epsilon_m$ small, although user preferences for rule complexity are likely to be more important in setting $\lambda_0$, $\lambda_1$. Given $\lambda_1$, the number of uniform samples $n$ could in turn be chosen to reduce $\epsilon_n$. Note that $\epsilon_m, \epsilon_n$ are largely controlled by $\lambda_1$ and depend only logarithmically on the dimension $d$.

%
%
\subsection{Estimation of \texorpdfstring{$\cB^\epsilon$}{Bepsilon} as Binary Classification}
\label{sec:overlaprules}
To estimate the set $\cB^\epsilon$ of inputs with bounded group propensity $\eta_t(X) := p(T=t\mid X)$, we follow in the tradition of using black-box (potentially non-parametric) estimators of propensity to identify overlapping or balanced cohorts in the study of causal effects~\citep{crump2009dealing,fogarty2016discrete}. This is typically done by fitting a classifier (e.g., logistic regression) for predicting $T$ given $X$, and letting $\hat{\eta}_t(x)$ be the estimated probability of class $t$ for input $x$.
Given such an estimate, we assign a label $\tilde{b}_i$ to each data point $x_i \in \cD$ indicating significant propensity for every group, 
\begin{equation}\label{eq:propensityestimator}
\forall i \in [m]: \tilde{b}_i = \prod_{t\in \cT} \mathds{1}[\hat{\eta}_t(x_i)\geq \epsilon]~.
\end{equation}
Let $\tcB = \{x_i : \tilde{b}_i = 1\}$. Similar to the case of $\cS^\alpha$, we may now reduce estimation of $\cB^\epsilon$ to binary classification. Given $\hcS$, the minimizer of~\eqref{eqn:binaryclass}, we again set up a Neyman-Pearson-like classification problem, now regarding the intersection $\hcS\cap\tcB$ as the positive class: 
\begin{align}\label{eqn:binaryclass2}
    \hcB := & \argmin_C \;\; \frac{1}{\lvert\hcS \setminus \tcB\rvert} \sum_{i: x_i\in\hcS \setminus \tcB} \mathds{1}[x_i \in \cC]  + R(\cC)
    \\
    & \text{subject to} \;\; \sum_{i: x_i\in\hcS\cap\tcB} \mathds{1}[x_i \in \cC] \geq \beta \lvert\hcS\cap\tcB\rvert~. \nonumber
\end{align}
The sets $\hcS \setminus \tcB$ and $\hcS\cap\tcB$ are defined by the solution to \eqref{eqn:binaryclass} and the base estimator \eqref{eq:propensityestimator}.
To accommodate the policy evaluation setting described in Section~\ref{sec:background}, we can modify the pseudo-labels defined in \eqref{eq:propensityestimator} to be $\tilde{b}_i(\pi) = \prod_{t \in \pi(x_i)} \mathds{1}[\hat{p}(T=t \mid X=x_i) \geq \epsilon]$, where $\pi(x_i) \coloneqq \{t: \pi(t | x_i) > 0\}$, and solve \eqref{eqn:binaryclass2} using $\tcB(\pi) = \{x_i : \tilde{b}_i(\pi) = 1\}$ in place of $\tcB$.  The resulting full procedure is given in the supplement.

\paragraph{Generalization of the final estimator.} In the supplement, we state and prove a theorem bounding the generalization error of our final estimator, $\hcO = \hcS \cap \hcB$. It shows that for good base estimators  $\hcS, \tilde{\cB}$, the error of $\hcO$ with respect to the true overlap $\cO$ is dominated by its error with respect to the base estimators. Hence, practitioners may make an informed tradeoff between accuracy and interpretability based on this metric.

%
%
\subsection{Optimizing Boolean Rules}
\label{sec:booleanrules}
Next, we describe a procedure for optimizing \eqref{eqn:binaryclass} over a class $\scrC$ of Boolean DNF rules. The same procedure also solves 
\eqref{eqn:binaryclass2}.

We assume that base features $X$ have been binarized to form literals such as $(\text{Age} > 30)$ or $(\text{Sex} = \text{Female})$, as is standard in e.g.~decision tree learning. 
A conjunction may thus be represented as the product of binary indicators of these literals. 
We let $\cK$ index the set of all possible (exponentially many) conjunctions of literals, e.g. $(\text{Age} > 30) \wedge \text{Female}$. Then, for $k\in \cK$, let $a_{ik} \in \{0,1\}$ denote the value taken by the $k$-th conjunction at sample $x_i$. Let the DNF rule be parametrized by $r \in \{0,1\}^{|\cK|}$ such that $r_k=1$ indicates that the $k$-th conjunction is used in the rule. 

Define an error variable $\xi_i$ for $i$ in $\cU \cup \cI$ representing the penalty for covering or failing to cover point $i$, depending on its set membership.
Then, problem \eqref{eqn:binaryclass} may be reformulated as follows, 
\begin{align}\label{eqn:overlap}
    & \underset{r}{\text{minimize}}\ \ \  \frac{1}{\lvert\cU\rvert} \sum_{i\in\cU} \xi_i + R(r) 
     \\ 
     & \text{subject to} \ 
     \left\{
     \begin{array}{ll}
     \displaystyle
      r_k \in \{0,1\}, \; k\in \cK, \\
     \displaystyle \xi_i \geq 1 - \sum_{k\in\cK} a_{ik} r_k, \;\; \xi_i \geq 0, \; i \in \cI,  \\
     \displaystyle \sum_{i\in\cI} \xi_i \leq (1-\alpha) m 
     \\
     \displaystyle \xi_i = \max_{k \in \cK}(a_{ik} r_k), \; i \in \cU.
    \end{array} \right. \nonumber
\end{align}
Problem \eqref{eqn:overlap} is an IP with an exponential number of variables and is intractable as written.  We follow the column generation approach of \cite{dash2018boolean} to effectively manage the large number of variables and solve \eqref{eqn:overlap} approximately. As in that previous work, we bound from above the $\max$ in the last constraint of \eqref{eqn:overlap} with the sum (Hamming loss instead of zero-one loss) as it gives better numerical results. 
The choice of regularization in \eqref{eqn:R} implies 
$
R(r) = \sum_{k \in \cK} \lambda_k r_k
$
with
$
\lambda_k = \lambda_0 + \lambda_1 p_k.
$
Thus the objective becomes linear in $r$, 
$
\sum_{k\in\cK} \left( 1/\lvert\cU\rvert \sum_{i\in\cU} a_{ik} +  \lambda_k \right) r_k
$, and the $\xi_i$, $i\in\cU$ constraints are absorbed into the objective. We then follow the overall procedure in \citep{dash2018boolean} of solving the linear programming (LP) relaxation, using column generation to add variables only as needed. 

We make the following departures from \cite{dash2018boolean}. As noted, \eqref{eqn:overlap} has a constraint on false negative rate instead of a corresponding objective term and a complexity penalty $R(r)$ while \cite{dash2018boolean} use a constraint.  As a result, the LP reduced costs, needed for column generation, are different. With dual variables $\mu_i \geq 0$, $i\in\cI$ corresponding to the $\xi_i$, $i\in\cI$ constraints in \eqref{eqn:overlap}, the reduced cost of conjunction $k$ is now
$
    1/\lvert\cU\rvert \sum_{i\in\cU} a_{ik} + \lambda_k - \sum_{i\in\cI} \mu_i a_{ik}$,
which remains a linear function of $a_{ik}$, allowing the same column generation method to be used.
We also avoid the need for an IP solver as used in \cite{dash2018boolean} by a) solving the column generation problem using a beam search algorithm from~\citep{wei2019generalized}, and b) restricting \eqref{eqn:overlap} to the final columns once column generation terminates, converting to a weighted set cover problem, and applying a greedy algorithm to obtain an integer solution.

\section{EXPERIMENTS}
\label{sec:experiments}

In our experiments, we seek to address the following questions, while relating the performance of OverRule to that of MaxBox (MB) \citep{fogarty2016discrete}, which is also designed to produce interpretable study populations. 
\begin{enumerate*}[label=(\roman*)]
    \item {\bf Why is support estimation important?} In Section~\ref{sec:irisexp} we give a conceptual illustration using the Iris dataset, where MaxBox returns a description that empirically includes a large space outside of the true overlap region.
    \item {\bf How well does OverRule approximate the base estimators / true overlap region?} In Section~\ref{sec:jobsexp} we use the Jobs \citep{lalonde1986evaluating} dataset to show that performance of OverRule is comparable to that of the base estimators, and generally surpasses the performance of MaxBox. 
    \item {\bf Do the resulting rules yield any insights?} We apply OverRule to overlap estimation in two real-world clinical datasets on (1) post-surgical opioid prescriptions, and (2) policy evaluation in antibiotic prescriptions.  For the former, we conducted a user study with three clinicians to interpret and critique the output, with additional comparison to the output of MaxBox.  
\end{enumerate*}

OverRule and MaxBox algorithms are both \textit{meta-algorithms} in the sense that they take (as input) labels indicating whether each data point is in the overlap set.  To generate these labels, we use a variety of base overlap estimators:
\begin{enumerate*}[label=(\roman*)]
    \item \textit{Covariate Bounding Boxes}: The intersection of covariate (marginal) bounding boxes (CBB), analogous to classical balance checks in causal inference. The bounding boxes are selected to cover the $[(1-\alpha)/2, (1+\alpha)/2]$ quantiles of the data. 
    \item \textit{Propensity Score Estimators}: Standard propensity score estimators as described in \eqref{eq:propensityestimator} and \citet{crump2009dealing} with logistic regression (PS-LR) or $k$-nearest neighbors (PS-$k$NN) estimates of the propensity.  These can be viewed as a binary version of overlap weights~\citep{li2018balancing}.
    \item \textit{One-Class SVMs}: One-Class Support Vector Machines (OSVM) to first estimate conditional supports and then use their intersection as overlap labels.
\end{enumerate*}
Details on hyperparameter selection and feature binarization are given in the supplement, along with general guidance on hyperparameter selection depending on user goals, from optimizing an observable metric (e.g., accuracy w.r.t the base estimator), to generating shorter rule sets, to exploring structure in the data.
%
%
\subsection{Illustrative Example: Iris}
\label{sec:irisexp}
We use the Iris dataset to illustrate the importance of combining explicit support estimation (lacking in MaxBox) with an interpretable characterization of the overlap region (lacking in propensity score models).  We use OverRule to identify the overlap between members of two species of Iris, as represented by their sepal and petal dimensions. In Figure~\ref{fig:iris2d}, we visualize the estimates $\hcO$ learned using OverRule and MaxBox in the space of sepal length and width. In contrast, the coefficients of a logistic regression propensity score model, $[-1.7, -1.5,  2.5,  2.6]^\top$ reveal very little about which points lie in the overlap set.

\begin{figure}[t!]
    \centering
    \includegraphics[width=.9\columnwidth]{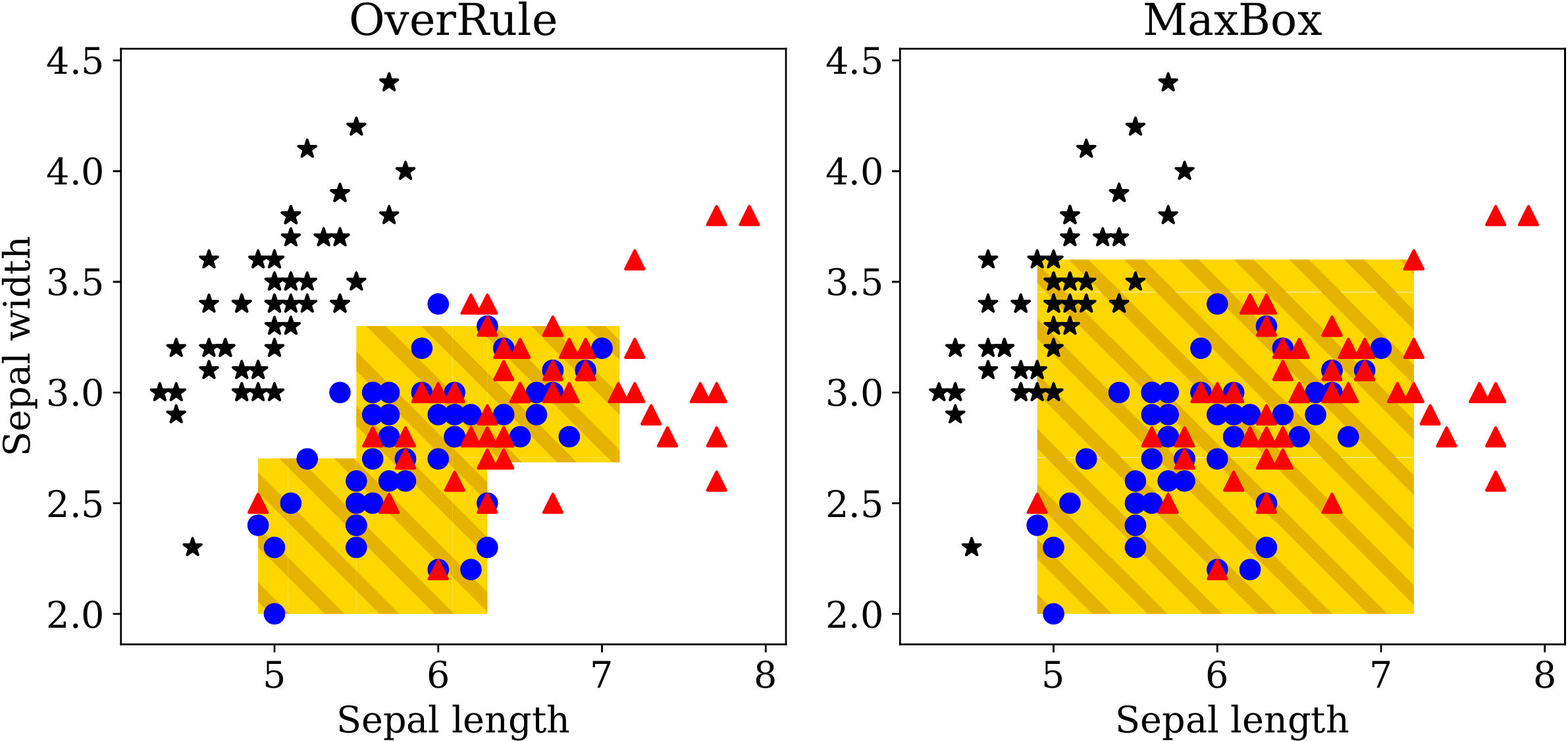}%
    \caption{Overlap (orange stripes) between Versicolor (blue circles) and Virginica (red triangles) species in the Iris dataset as identified by OverRule (left) and MaxBox (right) using the same base estimator of propensity. Black stars indicate samples of the (unobserved) Setosa species. We see that MaxBox identifies several of the Setosa samples as being in the overlap set, despite it being outside of the support of the observed data.}
    \label{fig:iris2d}
\end{figure}

%
%
\subsection{Job Training Programs}
\label{sec:jobsexp}
\begin{table}[t]
    \centering
    \caption{\label{tbl:results} Overlap estimation in Jobs. Balanced accuracy (Acc), false positive rate (FPR), false negative rate (FNR), and number of literals (L) with standard deviations over 5-fold CV. MB and OR indicate MaxBox and OverRule. MB did not run with CBB.}
    {\small
    \begin{tabular}{lcccc}
        \toprule
         & {\bf Acc} & {\bf FPR} & {\bf FNR} & {\bf L}\\ \midrule
        \multicolumn{4}{l}{\bf Baselines (base estimators):} \\ 
       CBB & $0.75 \pm 0.02$ & $0.12 \pm 0.01$  & $0.38 \pm 0.03$ & --- \\
       OSVM & $0.82 \pm 0.01$ & $0.22 \pm 0.03$  & $0.14 \pm 0.02$ & --- \\
       PS-$k$-NN & $0.90 \pm 0.02$ & $0.14 \pm 0.02$  & $0.05 \pm 0.02$ & --- \\
       PS-LR & $0.96 \pm 0.01$ & $0.10 \pm 0.01$  & $0.09 \pm 0.03$ & --- \\ \midrule
        \multicolumn{4}{l}{\bf MaxBox with base estimator:} \\ 
        OSVM & $0.68 \pm 0.01$ & $0.09 \pm 0.02$ & $0.54 \pm 0.01$ & 16 \\
        PS-$k$NN & $0.84 \pm 0.01$ & $0.03 \pm 0.01$ & $0.29 \pm 0.02$ & 16 \\
        PS-LR & $0.80 \pm 0.02$ & $0.04 \pm 0.01$ & $0.35 \pm 0.04$ & 16 \\ \midrule
        \multicolumn{4}{l}{\bf OverRule with base estimator:} \\ 
       CBB & $0.83 \pm 0.01$ & $0.16 \pm 0.01$ & $0.19 \pm 0.02$ & 20 \\
       OSVM & $0.84 \pm 0.02$ & $0.25 \pm 0.03$ & $0.07 \pm 0.02$ & 23 \\
       PS-$k$NN & $0.89 \pm 0.02$ & $0.16 \pm 0.02$ & $0.06 \pm 0.02$ & 40 \\
       PS-LR & $0.88 \pm 0.02$ & $0.15 \pm 0.04$ & $0.09 \pm 0.01$ & 21 \\
        \bottomrule
    \end{tabular}
    }
\end{table}
\begin{figure}
    \centering
    \includegraphics[width=0.9\columnwidth]{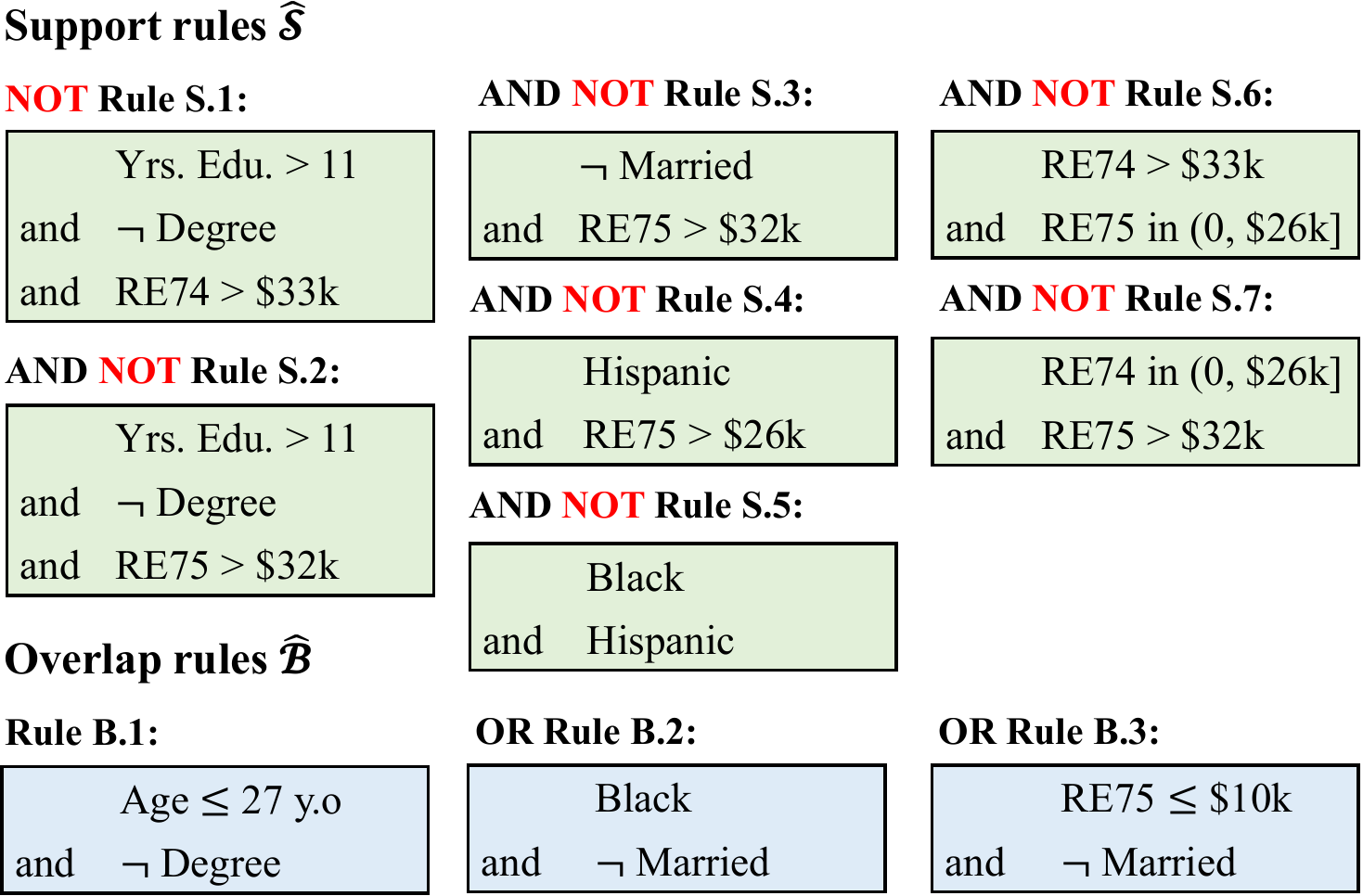}
    \caption{OverRule description of the overlap region $\cO$ in the Jobs dataset learned using the LR propensity base estimator, achieving held-out balanced accuracy of 0.88.  $\lnot$ indicates a negation, and CNF support rules are given with rule-level negations applied for readability. If \emph{none} of the support rules (top) and \emph{any} of the overlap rules (bottom) apply, a subject is in $\cO$.}
    \label{fig:jobsrules}
\end{figure}
In this section, we demonstrate that OverRule compares favorably to MaxBox in terms of approximating both the derived overlap labels (using a base estimator), as well as the ``ground truth'' overlap labels in a real dataset.  To do so, we use data from a famous trial performed to study the effects of job training~\citep{lalonde1986evaluating,smith2005does}, in which eligible US citizens were randomly selected into ($T=1$), or left out of ($T=0$) job training programs. The RCT ($E=1$), which satisfies overlap by definition, has since been combined with non-experimental control samples ($E=0, T=0$), forming a larger observational set (Jobs), to serve as a benchmark for causal effect estimation~\citep{lalonde1986evaluating}. Here, we aim to characterize the overlap between treated and control subjects. 

Due to the trial's eligibility criteria, the experimental and non-experimental cohorts barely overlap; standard logistic regression separates the experimental and non-experimental groups with held-out balanced accuracy of 0.96. Since all treated subjects were part of the experiment, the experimental cohort perfectly represents the overlap region. For this reason, we use the experiment indicator $E$ as ground truth for $\cO$, at the risk of introducing a small number of false negatives. In studies of causal effects in this data, the following features were included to adjust for confounding: Age, \#Years  of education (Educ), Race (black/hispanic/other), Married, No degree (NoDegr), Real earnings in 1974 (RE74) and 1975 (RE75). These are the features $X$ for which we estimate overlap. 

We present results in Table~\ref{tbl:results} and Figure~\ref{fig:jobsrules}, where all balanced accuracies are w.r.t.~the ground truth indicator $E$. For the propensity base estimators, the OverRule approximations achieve slightly lower balanced accuracies than the base estimator, but with a simpler description, while for the other base estimators the accuracy is actually better. OverRule compares favorably to MaxBox on balanced accuracy, although MaxBox generally achieves a lower FPR, likely because it does not try to retain a fixed fraction $\beta$ of the overlap set. In the supplement, we show that the held-out balanced accuracy quickly converges as the number of literals in the rules increases and correlates strongly with the quality by which the rule set approximates the base estimator. 

The learned support rules in Figure~\ref{fig:jobsrules} demonstrate that support estimation can find gaps in the dataset that are intuitive, such as a lack of individuals with high income but no degree (Rules S.1-2) or whose income changes dramatically from 1974 to 1975 (Rules S.6-7).  The learned overlap rules conform to expectations, as the eligibility criteria for the RCT allow only subjects who were currently unemployed and had been so for most of the time leading up to the trial---factors that correlate with age and education (Rule B.1), previous income (Rule B.3), and marital status (Rules B.2-3)~\citep{smith2005does}.

%
%
\subsection{Post-surgical Opioid Prescriptions}%
\label{sec:opioidexp}%

Opioid addiction affects millions of Americans. Understanding the factors that influence the risk of addiction is thus of great importance. To this end, \cite{brat2018postsurgical} and \cite{zhang2017exploring} study the effect of choices in opioid prescriptions on the risk of future misuse. Here, we study a group of \emph{post-surgical} patients who were given opioid prescriptions within 7 days of surgery, replicating the cohort eligibility criteria of~\cite{brat2018postsurgical} using a subset of the MarketScan insurance claims database. We compare groups of patients with morphine milligram equivalent (MME) doses above and below the 85th percentile in the cohort, MME=450. Subjects were represented by basic demographics (age, sex), diagnosis history, and procedures billed as surgical on the index date (not mutually exclusive). Cohort statistics are given in the supplement. 
\begin{figure}[t!]
    \centering
    \includegraphics[width=.90\columnwidth]{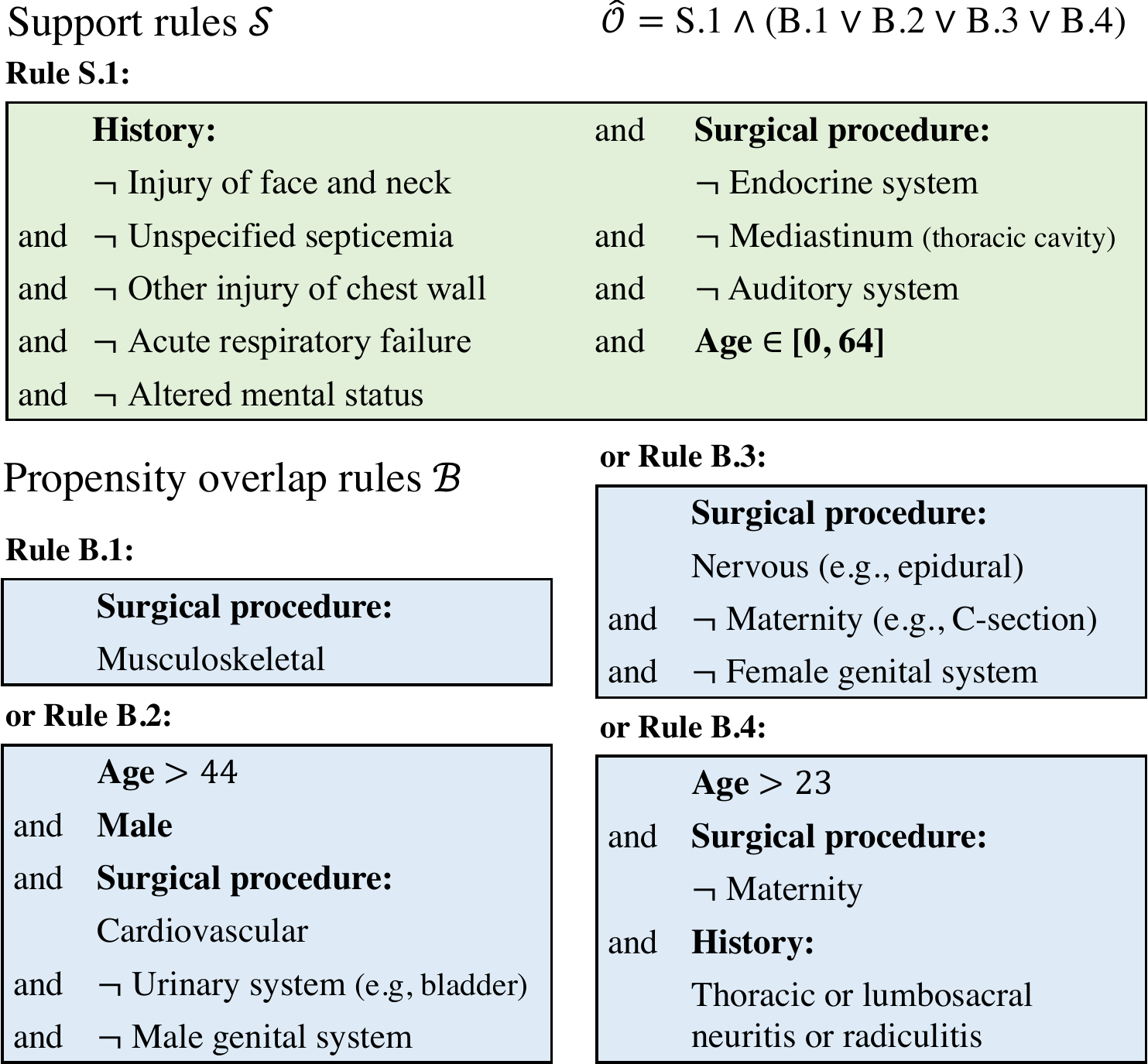}%
    \caption{OverRule description of post-surgical patients likely to receive both high and low opioid doses. A patient is in the overlap set if the support rule (top) applies and \emph{any} propensity overlap rule (bottom) applies. $\lnot$ indicates negation. The rules cover 27\% of patients with balanced accuracy of 0.90 w.r.t.~the base estimator.
    Surgical procedures are not mutually exclusive. \label{fig:opioid_rules}
    }
    \vspace{-2mm}
\end{figure}%
We fit three models: An OverRule model (OR) using DNF support rules and a random forest base estimator, a MaxBox model (MB)~\citep{fogarty2016discrete} with the same base estimator, and another OverRule model describing the complement of $\cO$ (OR-C).  The balanced accuracies of these models w.r.t. the base were 0.90 (OR), 0.77 (MB) and 0.92 (OR-C). Learning took 10 minutes for OverRule (Python) and 7 minutes for MaxBox (R). Other hyperparameter details are in the supplement.

In Figure~\ref{fig:opioid_rules}, we summarize the rules learned by OR which cover 27\% of the overall population.  MB learned: \small (Musculoskeletal surg. $\land$ $\lnot$Mediastinum surg. $\land$ $\lnot$Male genital surg.  $\land$ $\lnot$Maternity surg. $\land$ $\lnot$Lumbosacral spondylosis without myelopathy) \normalsize which covers 17\% of patients. The rules learned by OR-C are presented in the supplement. 

To evaluate the interpretability of the output, we conducted a qualitative user study through a moderated discussion with three participants: two attending surgeons (P1 \& P2) and a 4th year medical student (P3) at a large US teaching hospital. Before seeing the outputs of any method, the participants were asked to give their expectations for what to find in the overlap set. 

The participants expected that the overlap set would mostly correspond to patients in the higher dose range, as these patients are often considered also for smaller doses, and that overlap would be driven largely by surgery type. All participants expected {\rulet Musculoskeletal} and {\rulet Cardiovascular} surgery patients to be predominantly in the higher dose group, and sometimes in the lower, and one suggested that {\rulet Maternity} surgeries (e.g., C-sections) would be only in the lower range.  These comments are all consistent with the findings of OverRule, which identified all of these surgery types as important. MaxBox identified only {\rulet Musculoskeletal} surgery patients as overlapping. One participant expected history of psychiatric disease and {\rulet Tobacco use disorder} to be predictive of higher prescription doses for some patients, and thus overlap. Neither method identified psychiatric disease, but {\rulet Tobacco use disorder} was identified by OR-C as predictive (see supplement).

The participants found the support rules ($\hcS$) output by OR (Figure~\ref{fig:opioid_rules} top) intuitive. P1 stated that {\rulet Endocrine} surgeries are not typically followed by opioid prescriptions. They found the MaxBox and OR rule descriptions easy to interpret, and discussion focused on their clinical meaning. The first three propensity overlap rules B.1-B.3 were all consistent with expectation as described above, with the caveat that {\rulet Cardiovascular} patients are not typically stratified by {\rulet Urinary} and {\rulet Genital} surgeries. This was later partially explained by catheters being billed as {\rulet Urinary} and P3 interpreted this as a proxy for more severe {\rulet Cardiovascular} surgeries. P1 pointed out the value in discovering such surprising patterns that may be hidden in black-box analyses.
The OR-C rules were found hard to interpret due to many double negatives (``excluded from exclusion''), but were ultimately deemed clinically sound. 

\textbf{Remark}: We noted that these support rules primarily exclude individually rare features, in lieu of e.g., finding that certain non-rare surgery types do not co-occur.  This motivated both (1) an empirical study (w/semi-synthetic data) of how support rule hyperparameters influence the recovery of these interactions, and (2) the generation of new rules. Both are in the supplement.  
%
%
\subsection{Policy Evaluation of Antibiotic Prescription Guidelines}
\label{sec:uti}
Using the policy evaluation formulation of $\cB^\epsilon(\pi)$ (Section \ref{sec:overlaprulespolicy}), we apply OverRule to assess overlap for a policy that follows clinical guidelines published by the Infectious Disease Society of America (IDSA) for treatment of uncomplicated urinary tract infections (UTIs) in female patients \citep{Gupta2011}.  Using medical records from two academic medical centers, we apply OverRule to a cohort of 65,000 UTI patients to test whether it can recover a clinically meaningful overlap set.  From a qualitative perspective, we discussed the resulting rules with an infectious disease specialist, who verified that they have a clear clinical interpretation as identifying primarily outpatient cases and uncomplicated inpatient cases, which are where the guidelines are applied in practice.  Detailed results (including quantitative results) are given in the supplement. 

\section{CONCLUSION}%
\label{sec:discussion}%
We have presented OverRule---an algorithm for learning rule-based characterizations of overlap between populations, or the inputs for which policy evaluation from observational data is feasible. The algorithm learns to exclude points that are marginally out-of-distribution, as well as points where some population/policy has low density. We gave theoretical guarantees for the generalization of our procedure and evaluated the algorithm on the task of characterizing overlap in observational studies. These results demonstrated that our rule descriptions often have similar accuracy to black-box estimators and outperform a competitive baseline. In an application to study treatment-group overlap in post-surgical opioid prescription, a qualitative user study found the results interpretable and clinically meaningful. Similar observations were made in an application to evaluation of antibiotic prescription policies. Future research challenges include investigating the scalability of the method with the dimensionality of the input.

\subsubsection*{Acknowledgments} 
We thank Chloe O’Connell and Charles S. Parsons for providing clinical feedback on the opioid misuse experiment, and Sanjat Kanjilal for providing clinical feedback on the antibiotic prescription experiment. We also thank Bhanukiran Vinzamuri for assistance with the opioids data, David Amirault for insightful suggestions and feedback, and members of the Clinical Machine Learning group for feedback on earlier drafts. This work was partially supported by the MIT-IBM Watson AI Lab and the Wallenberg AI, Autonomous Systems and Software Program (WASP) funded by the
Knut and Alice Wallenberg Foundation.

\setcounter{table}{0}
\setcounter{figure}{0}
\setcounter{equation}{0}
\setcounter{footnote}{0}
\def\theequation{S\arabic{equation}}
\def\thetable{S\arabic{table}}
\def\thefigure{S\arabic{figure}}
\renewcommand{\thethmthm}{S\arabic{thmthm}}
\renewcommand{\thethmprop}{S\arabic{thmprop}}
\renewcommand{\thethmlem}{S\arabic{thmlem}}
\renewcommand{\thethmcol}{S\arabic{thmcol}}
\renewcommand{\thethmconj}{S\arabic{thmconj}}
\renewcommand{\thethmasmp}{S\arabic{thmasmp}}
\renewcommand{\thethmcorr}{S\arabic{thmcorr}}

\appendix
\section*{Supplementary Material}

The supplement is structured as follows:
\begin{itemize}
    \item {\bf Guidance on hyperparameter selection}: We take a deeper dive into the impact of hyperparameter selection on support and overlap estimation, including an in-depth empirical evaluation with concrete recommendations on how to set hyperparameters for support estimation given an a-priori belief that higher-order intersections of variables may be excluded from the cohort.
    \item {\bf Application to Policy Evaluation}: We discuss in more depth how the OverRule algorithm can be applied to finding areas of sufficient coverage for policy evaluation tasks.
    \item {\bf Additional experimental results}:  In addition to providing additional detail on the experiments presented in the main paper, we also present several results that were only alluded to in the main paper.  This includes the detailed results for the policy evaluation task (antibiotic prescription), as well as additional rules learned for the opioids prescription task.
    \item {\bf Theoretical results}: We include proofs for our theoretical results, as well as an additional Theorem bounding the generalization error of our two-stage estimator in terms of the error of the base estimators. 
\end{itemize}

In addition, to build further intuition for Boolean rules, we illustrate a Boolean rule in the DNF form in a 2D example in Figure~\ref{fig:notation}.
\begin{figure*}[ht]
    \centering
    \includegraphics[width=.9\textwidth]{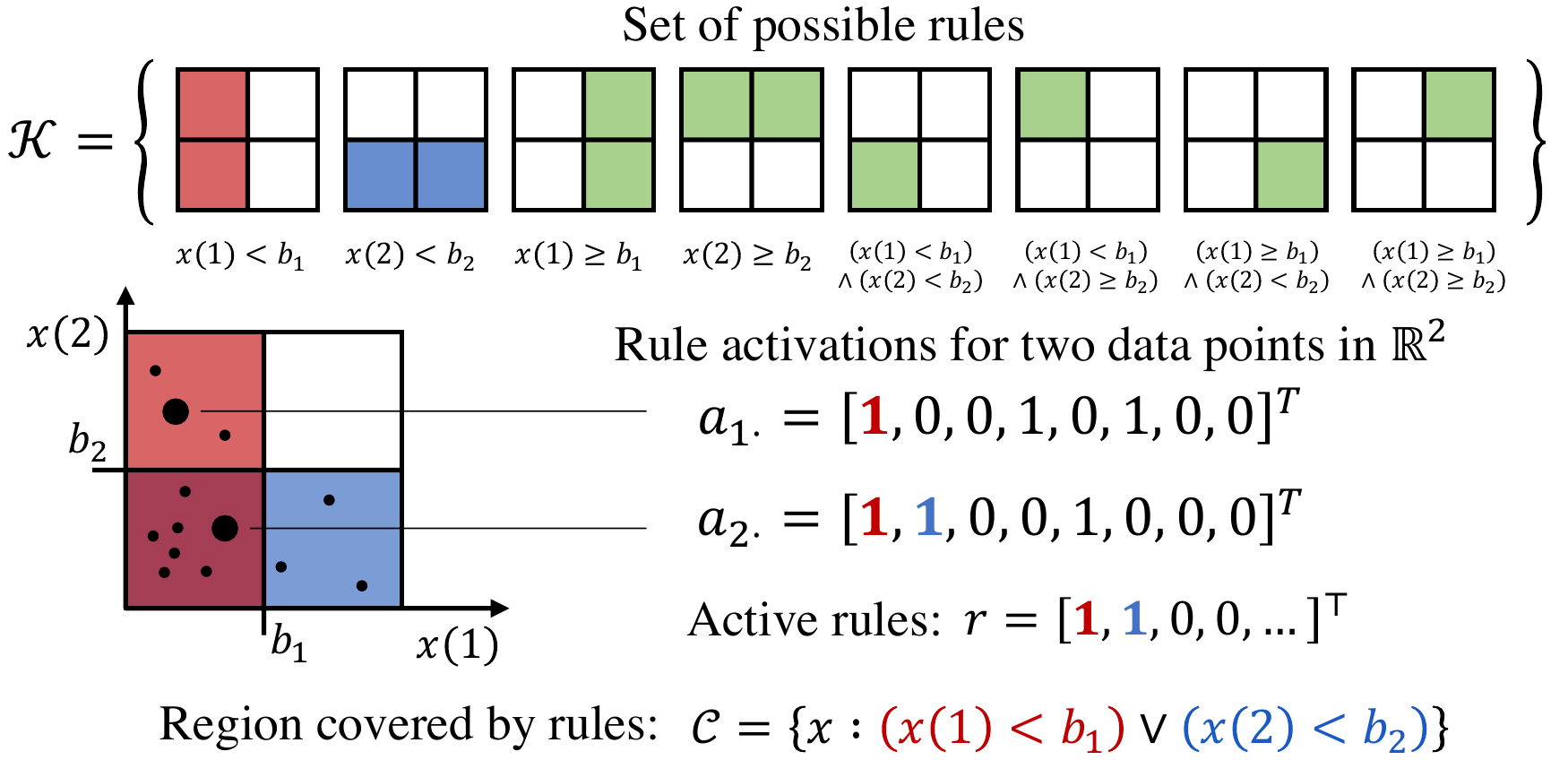}%
    \caption{Boolean rules on disjunctive normal form (DNF). We highlight data points represented by their activations, $a_{1\cdot}, a_{2\cdot}$ of rules from the set $\cK$ of all possible rules. $\cC$ is the region described by the rule set and $r$ indicators for the rules.}
    \label{fig:notation}
\end{figure*}

Code for this paper can be found at \url{https://github.com/clinicalml/overlap-code}

\section{Choosing Hyperparameters}%
\label{sec:hyperparameters}

\subsection{Overview}

Considering OverRule along with the base estimator, there are a few distinct sets of hyperparameters to choose
\begin{itemize}
    \item {\bf Support Rules}: The support rule estimation task requires a specification of DNF versus CNF form, a specification of $\alpha, \lambda_0, \lambda_1$ used in the objective, and the number of samples to draw from the reference measure.
    \item {\bf Base Estimator and Overlap Labels}: In addition to the hyperparameters of the base estimator itself, a threshold $\epsilon$ must be chosen to generate overlap labels
    \item {\bf Overlap Rules}: These rules similarly require a specification of DNF or CNF form, and specification of $\beta, \lambda_0, \lambda_1$.
\end{itemize}

For the base estimator itself, the hyperparameters can be tuned in the usual way using cross-validation using a metric of interest (e.g., AUC).  The choice of $\epsilon$ is studied in the existing literature \citep{crump2009dealing} and ultimately depends on the downstream causal inference task, though $\epsilon=0.1$ is sometimes considered as a rule of thumb.  For the support rules, we typically set the number of reference measure samples to be as large as computationally feasible.

For the overlap and support rules, the remaining hyperparameters can be chosen (1) by using cross-validation to optimize for balanced accuracy (or some other metric, like false positive rates) with respect to the overlap labels or uniform background samples, (2) with some other objective in mind, e.g., setting the $\lambda$ parameters to be large to discourage many rules, even if more rules would increase accuracy, or (3) with the goal in mind of choosing values (or exploring a range of values) most likely to discover ``interesting patterns'' in the cohort.

We expand upon a concrete instance of this latter goal in the remainder of this section, particularly as regards hyperparameter selection for support estimation, where extremely high accuracy is particularly easy to achieve and is thus less informative for the purposes of hyperparameter selection.

\subsection{Choosing Support Hyperparameters to highlight exclusions}%
\label{sub:hyperparameter_support}

{\bf Motivation}: In the context of our motivating applications, the primary purpose of support estimation is to identify regions where we do not have any (or have very few) observations.  For instance, if there are no men in our dataset who also have cardiac arrhythmia\footnote{This would be surprising, as men with arrhythmia are fairly common in the general population}, then this would be a clinically relevant fact that should be highlighted.  Thus, we would like to select hyperparameters which minimize our risk of overlooking these types of exclusions.  

In this section we give some guidance on how to select hyper-parameters for support estimation with this particular goal in mind, based on synthetic and real-data experiments.  To recap, these hyper-parameters include
\begin{enumerate*}[label= (\roman*)]
  \item $\alpha$, the support level, and 
  \item $\lambda_0, \lambda_1$, regularization parameters for learning support rules.
\end{enumerate*}
There are also relevant hyperparameters in the underlying algorithm of \citet{wei2019generalized}, primarily the width of the beam search used during column generation.

\textbf{Summary:} For this purpose, we recommend setting $\alpha \approx 1$, and in particular we consistently observed best results for $\alpha \geq 0.98$. We observe that for $\alpha$ sufficiently close to 1, the results are less sensitive to different values of $\lambda_0, \lambda_1$.  In addition, we recommend setting the width of the beam search in the algorithm of \citet{wei2019generalized} to be on the same order of magnitude as the number of binary features.  

These recommendations have the effect of encouraging the algorithm to consider higher-order interactions between variables that describe regions with little or no support in the data (e.g., \enquote{there are no men with cardiac arrhythmia}), and we verify this through experiments where we selectively remove regions of the data, and verify whether or not the algorithm can recover these regions. 

Concretely, we use both a synthetic and semi-synthetic case where we manually exclude all points which satisfy a simple boolean rule, and look to identify that exclusion automatically.  That is to say, in both cases we take a dataset and \textbf{remove} data points $\bx \in {\{0, 1\}}^d$ which satisfy a rule of the form $x_i = 1 \land x_j = 1$ for two features $x_i, x_j$, and then check if our algorithm incorporates this into the learned rule set. 
\begin{itemize}
  \item \textbf{Synthetic Case:} In this setting, we generate data comprised of 22 independent binary features, such that 10 features are rare (binomial with $p=0.01$), 12 features are common ($p=0.5$), and we remove all data points which satisfy a conjunction of the last two common features.
  \item \textbf{Semi-Synthetic Case (Antibiotic Prescription):} In this setting, we used the medical records dataset described in Section~\ref{sec:uti}, and removed all men with cardiac arrhythmia, which compromised 5\% of the total population.
\end{itemize}

This particular type of exclusion benefits from a CNF formulation (AND of ORs) of the support task.  This is because the exclusion can be described in a parsimonious way (independently of other aspects of support) as a single additional rule.  As discussed in Section~\ref{sec:MVruleSet}, it is straightforward to convert the CNF formulation to a DNF formulation and vice versa.  However, we note that the CNF formulation (for a fixed number of reference samples) can be more computationally intensive than the DNF formulation.

\subsubsection{Synthetic Experiments}%
\label{sub:synthetic_exp}

For the synthetic case, our goal is to build intuition that we can validate in the semi-synthetic setting.  We will first describe our data-generating process in more detail, and then describe the results and conclusions from an exhaustive hyperparameter search.

\textbf{Synthetic Data Generation:} We generate data as follows.  Note that we are only concerned (for the moment) with estimating support, so we do not include any notion of treatment groups.
\begin{itemize}
  \item We sample 10,000 data points $x \in {\{0, 1\}}^{d}$ where $d = 22$, by sampling (for each data point):
    \begin{itemize}
      \item 10 \enquote{rare} binary features $r_1, \ldots, r_{10}$, generated independently with $p = 0.01$
      \item 12 \enquote{common} binary features $c_{1}, \ldots, c_{12}$, generated independently with $p = 0.01$
      \item Thus, each data point is given by $\bx = [r_1, \ldots, r_{10}, c_1, \ldots, c_{12}]$
    \end{itemize}
  \item We remove all data points which satisfy $c_{11} = 1 \land c_{12} = 1$, which is approximately 25\% of all data points.  Our goal is to recover the corresponding \textbf{inclusion rule} as part of the final rule set of $c_{11} = 0 \lor c_{12} = 0$.
\end{itemize}

\textbf{Hyperparameter Search \& Outcomes:} With this setup, we estimate support using the algorithm given in the main paper, using every combination of the following hyperparameters
\begin{itemize}
  \item $\alpha \in \{0.95, 0.96, 0.97, 0.98, 0.99\}$, the constraint on covering our data.
  \item $\lambda_0 \in \{0, 10^{-6}, 10^{-4}, 10^{-2}\}$, and $\lambda_1 \in \{10^{-6}, 10^{-4}, 10^{-2}\}$, the regularization terms.
  \item $B \in \{10, 15, 20, 25, 30\}$, the width of the beam search used in \citet{wei2019generalized}
\end{itemize}
For each combination of hyperparameters, we run the experiment three times, generating a new set of fake data with each run.  The same three random seeds are used across all hyperparameter combinations.  We recorded a number of relevant outcomes, including 
\begin{itemize}
  \item Does the final rule set include the inclusion rule $c_{11} = 0 \lor c_{12} = 0$?
  \item How many rules are considered in the final rule set, and how long (on average) are these rules?
  \item How many \enquote{perfect} rules are found, which exclude none of the generated data points?
\end{itemize}

\textbf{Observations:} The full results of the hyperparameter search are given in Table~\ref{tab:synth_full_results}, but we summarize our observations and recommendations here.
\begin{itemize}
  \item \textit{Recovery by LP $ \rightarrow $ recovery by rounded rules:} Across all hyperparameter settings, if the desired inclusion rule was found during column generation (and thus considered by the LP), it was uniformly included in the final rounded rule.\footnote{This is not a general rule; While it holds in the synthetic case, it will not hold exactly in the semi-synthetic case with real data, as demonstrated in the next section.} Thus, our goal is to ensure that the desired inclusion is picked up by the LP during column generation.
  \item \textit{Beam Search Width should be higher than \# features:}  Recall that the LP relaxation with column generation starts by considering only rules with a single literal, and beam search is used to select additional rules for consideration, with a maximum width of $B$.  If $B$ is lower than the number of rare features, then the first $B$ rules considered will tend to be rules on single rare features.  This prevents the beam search from exploring interactions between more prevalent features.  Setting the beam-search width to a sufficiently high number ($\approx$ total features) forces the column generation to explore all rules with two literals, helpful for recovery of our desired inclusion rule.  This is demonstrated in Table~\ref{tab:set_b}.
  \item \textit{Higher values of $\alpha$ produce more stable results across $\lambda$}. Higher values of $\alpha$ tends to render the results less sensitive to choice of regularization $\lambda$, and tends to produce more reliable results in terms of recovery of our desired rule.  As demonstrated in Tables~(\ref{tab:set_a1}-\ref{tab:set_a3}), lower values of $\alpha$ are more sensitive to $\lambda_1$ in terms of both recovering the desired exclusion, as well as the number of rules found.  At higher values of $\alpha$, there is more consistent recovery of \enquote{perfect} rules, which exclude none of the sample points (and hence do not contribute to the constraint).
\end{itemize}
\begin{table}[h]
  \centering
  \caption{Beam Search Width and proportion of runs (across all other hyperparameter settings of $\alpha, \lambda_0, \lambda_1$) in which the synthetic region was correctly identified by the final rule set (\enquote{Rounded}).  Once the beam search width is sufficiently high (larger than the number of rare features), further increasing it does not appear to help.}%
  \label{tab:set_b}%
  \begin{tabular}{lccccc}
    Beam Width & 10 & 15 & 20 & 25 & 30 \\
    \midrule
    Recovered &0.07 & 0.87 & 0.87 & 0.87 & 0.87\\
  \end{tabular}
\end{table}
\begin{table}[h]
\centering
\begin{subtable}[h]{1.0\linewidth}
    \centering
    \caption{Recovery of inclusion rule}
    \label{tab:set_a1}%
    \begin{tabular}{l|ccc}
      & $\lambda_1=$1e-6 & $\lambda_1=$1e-4 & $\lambda_1=$1e-2 \\
    \midrule
$\alpha=0.95$ & 1.0 & 1.0 & 0 \\
$\alpha=0.96$ & 1.0 & 1.0 & 0 \\
$\alpha=0.97$ & 1.0 & 1.0 & 1.0 \\
$\alpha=0.98$ & 1.0 & 1.0 & 1.0 \\
$\alpha=0.99$ & 1.0 & 1.0 & 1.0 \\
    \end{tabular}
  \end{subtable}
\begin{subtable}[h]{1.0\linewidth}
    \centering
    \caption{Avg. \# of rules}
    \label{tab:set_a2}%
    \begin{tabular}{l|ccc}
    & $\lambda_1=$1e-6 & $\lambda_1=$1e-4 & $\lambda_1=$1e-2 \\
    \midrule
$\alpha=0.95$ & 23.67 & 15.75 & 5.0\\
$\alpha=0.96$& 35.58 & 33.33 & 4.0\\
$\alpha=0.97$& 39.83 & 31.92 & 4.0\\
$\alpha=0.98$& 44.17 & 47.17 & 23.83\\
$\alpha=0.99$& 31.42 & 31.25 & 27.67\\
    \end{tabular}
  \end{subtable}
\begin{subtable}[h]{1.0\linewidth}
    \centering
    \caption{Avg. \# of Perfect Rules}
    \label{tab:set_a3}%
    \begin{tabular}{l|ccc}
    & $\lambda_1=$1e-6 & $\lambda_1=$1e-4 & $\lambda_1=$1e-2 \\
    \midrule
 $\alpha=0.95$&  12.5 & 9.25 & 0.0\\
$\alpha=0.96$& 20.75 & 18.67 & 0.0\\
$\alpha=0.97$& 24.67 & 24.92 & 1.0\\
$\alpha=0.98$& 30.17 & 28.33 & 14.0\\
$\alpha=0.99$& 23.0 & 24.08 & 20.42\\
  \end{tabular}
  \end{subtable}
\caption{Value of $\alpha$ parameter and $\lambda_1$ parameters, for a fixed beam search width ($B = 15$), along with (a) the proportion of runs (across all other hyperparameter settings) in which the synthetic region was correctly identified by the final rule set, (b) the number of rules in the final solution, and (c) the number of perfect rules, defined as those which exclude none of the samples but which exclude some number of reference points.  Note that these results marginalize over $\lambda_0$, and (b-c) are averaged across all runs.}%
\label{fig:set_a}
\end{table}

\textbf{Discussion / Intuition:} Due to the greedy nature of the column generation procedure, a common failure mode is to only consider rules that include rare features, because those singleton rules exclude a significant amount of reference measure, and excluding rare features does not violate the $\alpha$-constraint.  For instance, a support rule of the form \enquote{not one of these K rare features} will (roughly speaking) exclude $K$ percent of the samples (if each rare feature has 1\% prevalence), while producing a volume of $2^{-K}$.  Thus, an overly greedy approach can obtain an objective value that is exponentially small in the number of rare features excluded, as long as it does not hit the $\alpha$ constraint.  This has the effect of \enquote{crowding out} more complex rules. 

Take a concrete example in Table~\ref{tab:set_a2} to build intuition for how the greedy set covering algorithm can fail in this case: Suppose $\lambda_0 = 0$, $\lambda_1 = 0.01$, and $\alpha = 0.95$, and suppose that our current solution excludes $5$ rare features before hitting the $\alpha$ constraint, then the reference volume is given by $2^{-5} \approx 0.03$.  In this case, adding the desired inclusion rule will reduce the volume by $1/4$ (a reduction in absolute terms which is $< 0.01$) while increasing the regularization penalty by $0.02$.  Thus, it will not be included.

To avoid this failure mode, we can increase $\alpha$, which has the effect of reducing the number of singleton rules $K$ that can be added before violating the constraint.  

\subsubsection{Semi-Synthetic Experiments}%
\label{sub:semi_synthetic_exp}

In the semi-synthetic experiment, our goal is to verify that the intuition from the synthetic setting carries over to a real dataset.

\textbf{Semi-Synthetic Data Generation:} We generated the dataset for this experiment as follows.  
\begin{enumerate}
  \item \textit{Subsampling:} We randomly sample 5000 patients from the full cohort of 65k patients, due to computational constraints.  In this subset, there were 185 binary features, and 5 continuous features.
  \item \textit{Synthetic Exclusion:} We remove all male patients with cardiac arrhythmia, which was around 5\% of the total population.
  \item \textit{Pre-Processing:} Given the prevalence of very rare binary features, we removed all binary features with a prevalence of less than 1\%, as well as all samples that had any of these features, resulting in the removal of 118 binary features and ~850 samples.  This was done both for computational reasons (to reduce the number of features) as well as to condition the problem such that it is more realistic for the support estimation to recover higher-order interactions.
  \item \textit{Final Dataframe}: The final dataset had 66 binary features and 5 continuous features, with the latter being converted into binary features via the use of deciles.
\end{enumerate}

\textbf{Hyperparameter Search:} We then followed a similar approach to the synthetic experiment, using every combination of the following hyperparameters.  For each combination, we ran the algorithm three times, inducing randomness over the data by taking a random 80\% of the data with each iteration.
\begin{itemize}
  \item $\alpha \in \{0.95, 0.96, 0.97, 0.98, 0.99\}$
  \item $\lambda_0 \in \{10^{-6}, 10^{-4}, 10^{-2}\}, \lambda_1 \in \{10^{-6}, 10^{-4}, 10^{-2}\}$
\end{itemize}

In this case, we fixed the width of the beam search at $B=1000$ (which encourages a more thorough search during column generation, as discussed above), and also found that we needed to adjust the value of $K$, another hyperparameter from the column generation algorithm, to be roughly on the same order as $B$.  The parameter $K$ controls how many rules get added to the LP at each iteration.  We also fixed the maximum number of iterations at 10.  We recorded all the same outcomes as were used in the synthetic case.

\textbf{Observations}: We observed that a number of patterns from the synthetic case carried over to the semi-synthetic case.
\begin{itemize}
  \item \textit{Inclusion in LP (mostly) implies inclusion in final rules:} When the desired inclusion rule appears among the rules considered during column generation, it mostly appears in the final rounded rules, in ~80\% of runs.  We conjecture that this is due to a large number of \enquote{perfect} rules existing in this dataset, which are also two-variable interactions, though many of these appear to be noise (see example inclusion rules below).
  \item \textit{Increasing $\alpha$ leads to more consistent recovery in the LP} of the desired inclusion rule.  However, as discussed, this does not always translate into the desired inclusion rule showing up in the final rounded rule set.  See Table~\ref{tab:set_a_abx}
  \item \textit{Higher values of $\alpha$ are less sensitive to choice of $\lambda$}: In Tables~(\ref{tab:set_a1_abx}-\ref{tab:set_a2_abx}) we demonstrate that, similar to the synthetic case, the number of rules and the number of \enquote{perfect} rules is highly sensitive to $\lambda_1$ when $\alpha$ is lower, but for $\alpha \geq 0.98$ it yields consistent results across different values of $\lambda$.
\end{itemize}

\textbf{Example \enquote{Perfect} Rules}: These rules exclude none of the samples in our data, while excluding reference points.  While occasional rules appear to be based on reasonable exclusions (such as a lack of pregnant veterans, given that 80\% of veterans are male in our data), most appear to be combinations of rare features (such as rare medications) that simply do not appear together in our data.  These are three representative rules from one run (where $\alpha=0.99, \lambda_0=\lambda_1=1e-6$, resulting in 23 rules, of which 17 were \enquote{perfect}):
\begin{itemize}
  \item not (Pregnant and Veteran)  
  \item not (Complicated Hypertension and Previous Medication of Cephalexin)
  \item not (Previous Medication of Doxycycline and Norfloxacin)
\end{itemize}

\begin{table}[h]
  \centering
  \caption{Values of $\alpha$ and the proportion of runs in which the desired inclusion rule was included in the LP during column generation, as well as included in the final rule set.  Results are averaged over values of $\lambda_0, \lambda_1$, with the exception of $\lambda_0=\lambda_1=1e-2$, because this did not run for $\alpha=0.97$}%
  \label{tab:set_a_abx}
  \begin{tabular}{l|cc}
    & LP & Final Rule Set\\
    \midrule
 $\alpha=0.95$ &  0.50 & 0.50 \\
 $\alpha=0.96$ &  0.75 & 0.71 \\
 $\alpha=0.97$ &  1.00 & 0.88 \\
 $\alpha=0.98$ &  1.00 & 0.62 \\
 $\alpha=0.99$ &  1.00 & 0.62 \\
  \end{tabular}
\end{table}

\begin{table}
\centering
\begin{subtable}[h]{1.0\linewidth}
    \centering
    \caption{Recovery of inclusion rule}
    \label{tab:set_a1_abx}%
    \begin{tabular}{l|ccc}
      & $\lambda_1=$1e-6 & $\lambda_1=$1e-4 & $\lambda_1=$1e-2 \\
    \midrule
$\alpha=0.95$ & 0.7 & 1.0 & 0.0 \\ 
$\alpha=0.96$ & 1.0 & 0.8 & 0.0 \\
$\alpha=0.97$ & 0.8 & 0.7 & 1.0 \\
$\alpha=0.98$ & 0.7 & 0.5 & 0.7 \\ 
$\alpha=0.99$ & 0.8 & 0.7 & 0.3 \\ 
    \end{tabular}
  \end{subtable}
\begin{subtable}[h]{1.0\linewidth}
    \centering
    \caption{Avg. \# of rules}
    \label{tab:set_a2_abx}%
    \begin{tabular}{l|ccc}
    & $\lambda_1=$1e-6 & $\lambda_1=$1e-4 & $\lambda_1=$1e-2 \\
    \midrule
$\alpha=0.95$ &210.2 & 115.8 & 6.0 \\
$\alpha=0.96$& 334.3 & 148.0 & 5.0 \\
$\alpha=0.97$& 25.2 & 75.2 & 49.8 \\ 
$\alpha=0.98$& 25.0 & 24.7 & 24.3 \\ 
$\alpha=0.99$& 23.3 & 23.3 & 23.7 \\ 
    \end{tabular}
  \end{subtable}
\begin{subtable}[h]{1.0\linewidth}
    \centering
    \label{tab:set_a3_abx}%
    \caption{Avg. \# of Perfect Rules}
    \begin{tabular}{l|ccc}
    & $\lambda_1=$1e-6 & $\lambda_1=$1e-4 & $\lambda_1=$1e-2 \\
    \midrule
 $\alpha=0.95$&200.2 & 105.8 & 0.0 \\ 
$\alpha=0.96$& 326.0 & 140.0 & 0.0 \\
$\alpha=0.97$& 19.5 & 69.0 & 42.2 \\
$\alpha=0.98$& 21.3 & 21.0 & 20.7 \\
$\alpha=0.99$& 19.0 & 18.7 & 19.7 \\
  \end{tabular}
  \end{subtable}
\caption{Value of $\alpha$ parameter and $\lambda_1$ parameters, along with (a) the proportion of runs (across all other hyperparameter settings) in which the synthetic region was correctly identified by the final rule set, (b) the number of rules in the final solution, and (c) the number of perfect rules, defined as those which exclude none of the samples but which exclude some number of reference points.  Note that these results marginalize over $\lambda_0 \in \{1e-6, 1e-4\}$ because $\lambda_0=\lambda_1=1e-2$ did not run for $\alpha=0.97$, and (b-c) are averaged across all runs.}%
\label{fig:set_a_abx}
\end{table}

\begin{table}[ht]
  \tiny
  \centering
  \caption{\textbf{Rec}: Proportion of runs where synthetic exclusion was recovered.  \textbf{\# R}: Number of rules in final output.  \textbf{\# PR}: Number of \enquote{perfect} rules which exclude zero data points. \textbf{Length}: Average length of rules.  Each entry is the average of three independent runs with different random seeds, and run with $B=15$}%
  \label{tab:synth_full_results}
  \begin{tabular}{ccc|cccc}
    $\alpha$ & $\lambda_0$ & $\lambda_1$ & Rec & \# R & \# PR & Length \\
    \toprule
    0.95 & 0     & 1e-06 & 1.00 & 31.00 & 17.00 & 2.36\\
         &      & 1e-04 & 1.00 & 19.33 & 12.00 & 2.25\\
         &      & 1e-02 & 0.00 & 5.00 & 0.00 & 1.00\\
         & 1e-06 & 1e-06 & 1.00 & 30.67 & 17.00 & 2.37\\
         &       & 1e-04 & 1.00 & 19.33 & 12.00 & 2.25\\
         &       & 1e-02 & 0.00 & 5.00 & 0.00 & 1.00\\
         & 1e-04 & 1e-06 & 1.00 & 27.00 & 15.00 & 2.36\\
         &       & 1e-04 & 1.00 & 18.33 & 12.00 & 2.23\\
         &       & 1e-02 & 0.00 & 5.00 & 0.00 & 1.00\\
         & 1e-02 & 1e-06 & 1.00 & 6.00 & 1.00 & 1.17\\
         &       & 1e-04 & 1.00 & 6.00 & 1.00 & 1.17\\
         &       & 1e-02 & 0.00 & 5.00 & 0.00 & 1.00\\
  \midrule
  0.96 & 0     & 1e-06 & 1.00 & 46.33 & 28.33 & 2.69\\
       &       & 1e-04 & 1.00 & 43.67 & 25.00 & 2.43\\
       &       & 1e-02 & 0.00 & 4.00 & 0.00 & 1.00\\
       & 1e-06 & 1e-06 & 1.00 & 45.33 & 27.67 & 2.70\\
       &       & 1e-04 & 1.00 & 43.67 & 25.67 & 2.41\\
       &       & 1e-02 & 0.00 & 4.00 & 0.00 & 1.00\\
       & 1e-04 & 1e-06 & 1.00 & 45.67 & 26.00 & 2.67\\
       &       & 1e-04 & 1.00 & 41.00 & 23.00 & 2.41\\
       &       & 1e-02 & 0.00 & 4.00 & 0.00 & 1.00\\
       & 1e-02 & 1e-06 & 1.00 & 5.00 & 1.00 & 1.20\\
       &       & 1e-04 & 1.00 & 5.00 & 1.00 & 1.20\\
       &       & 1e-02 & 0.00 & 4.00 & 0.00 & 1.00\\
  \midrule
  0.97 & 0     & 1e-06 & 1.00 & 49.67 & 31.00 & 2.74\\
       &       & 1e-04 & 1.00 & 38.00 & 30.00 & 2.51\\
       &       & 1e-02 & 1.00 & 4.00 & 1.00 & 1.25\\
       & 1e-06 & 1e-06 & 1.00 & 49.67 & 31.00 & 2.73\\
       &       & 1e-04 & 1.00 & 38.00 & 30.00 & 2.51\\
       &       & 1e-02 & 1.00 & 4.00 & 1.00 & 1.25\\
       & 1e-04 & 1e-06 & 1.00 & 48.33 & 29.00 & 2.71\\
       &       & 1e-04 & 1.00 & 37.33 & 29.33 & 2.55\\
       &       & 1e-02 & 1.00 & 4.00 & 1.00 & 1.25\\
       & 1e-02 & 1e-06 & 1.00 & 11.67 & 7.67 & 2.27\\
       &       & 1e-04 & 1.00 & 14.33 & 10.33 & 2.43\\
       &       & 1e-02 & 1.00 & 4.00 & 1.00 & 1.25\\
  \midrule
  0.98 & 0     & 1e-06 & 1.00 & 47.00 & 33.67 & 2.82\\
       &       & 1e-04 & 1.00 & 50.67 & 30.33 & 2.74\\
       &       & 1e-02 & 1.00 & 27.33 & 16.00 & 1.97\\
       & 1e-06 & 1e-06 & 1.00 & 46.67 & 33.33 & 2.81\\
       &       & 1e-04 & 1.00 & 50.67 & 30.33 & 2.74\\
       &       & 1e-02 & 1.00 & 27.00 & 15.67 & 1.97\\
       & 1e-04 & 1e-06 & 1.00 & 46.00 & 31.33 & 2.74\\
       &       & 1e-04 & 1.00 & 50.67 & 31.00 & 2.74\\
       &       & 1e-02 & 1.00 & 28.00 & 16.33 & 1.99\\
       & 1e-02 & 1e-06 & 1.00 & 37.00 & 22.33 & 2.29\\
       &       & 1e-04 & 1.00 & 36.67 & 21.67 & 2.26\\
       &       & 1e-02 & 1.00 & 13.00 & 8.00 & 1.95\\
  \midrule
  0.99 & 0     & 1e-06 & 1.00 & 33.00 & 23.33 & 2.33\\
       &       & 1e-04 & 1.00 & 33.00 & 27.33 & 2.33\\
       &       & 1e-02 & 1.00 & 28.33 & 21.00 & 1.96\\
       & 1e-06 & 1e-06 & 1.00 & 33.00 & 21.67 & 2.36\\
       &       & 1e-04 & 1.00 & 34.33 & 24.67 & 2.30\\
       &       & 1e-02 & 1.00 & 28.33 & 21.00 & 1.96\\
       & 1e-04 & 1e-06 & 1.00 & 31.33 & 25.67 & 2.34\\
       &       & 1e-04 & 1.00 & 27.00 & 20.67 & 2.17\\
       &       & 1e-02 & 1.00 & 28.33 & 21.00 & 1.96\\
       & 1e-02 & 1e-06 & 1.00 & 28.33 & 21.33 & 2.08\\
       &       & 1e-04 & 1.00 & 30.67 & 23.67 & 2.11\\
       &       & 1e-02 & 1.00 & 25.67 & 18.67 & 1.96\\
  \end{tabular}
\end{table}

\clearpage
\section{Application of OverRule to Policy Evaluation}
\label{sec:policyevalappendix}
\label{sec:policytheory}

In this section we give the detailed algorithm for applying OverRule to policy evaluation, as described in the main paper.  In this context, we wish to evaluate not a specific treatment decision (e.g., the average treatment effect of giving a drug vs. withholding it), but rather a conditional \textit{policy} representing a personalized treatment regime, which we will refer to as the \textit{target} policy.  This problem falls under the setting of off-policy policy evaluation when this target policy $\pi$ differs from the policy which generated the data, which we observe in the observational data as $p(T = t \mid x)$.

\textbf{Rationale for $\cB^\epsilon(\pi)$}: In the main paper, we drew a connection between the set $\cB^\epsilon$ and the following set, a function of the target policy $\pi$, 
$
\cB^\epsilon(\pi) := \{x \in \cX ; \forall t : \pi(t\mid x) > 0 :  p(T=t \mid x) > \epsilon\}
$. In this section, we recall the theoretical rationale for why we are restricted to this set, if we wish to evaluate the policy $\pi$ given samples generated according to $p(T = t \mid x)$.

Following similar notation to \cite{pmlr-v84-kallus18a}, we will let $X \in \cX$ correspond to covariates, $Y \in \cY$ to an outcome of interest, $T \in \cT$ to a treatment decision.  We write $\pi(t | x_i)$ as the probability of each treatment under the policy, which may be stochastic.  We write $Y(t)$ to represent the potential outcome under treatment $t$.  In this setting, we wish to evaluate the expected value of $Y$ under the target policy, which we denote as $\E[Y(\pi)]$. 
\begin{thmprop}[Informal]
    The expectation $\E[Y(\pi)]$ is only defined w.r.t. the observed distribution $p(X, T, Y)$ for the subset $B \in \cX$ such that $\forall x \in B, \ \pi(T = t \mid X = x) > 0 \implies p(T = t \mid X = x) > 0$
    \label{thmprop:coverage}
\end{thmprop}
\begin{proof}
Under the assumption that ignorability \citep{pearl2009causality} holds, we can write out our desired quantity as follows in terms of observed distribution $p(X, T, Y)$. For brevity, let $p(t\mid x) = p(T=t\mid X=x), p(x) = p(X=x)$, et cetera.
\begin{align}
    &\E[Y(\pi)] \\
    =&\int_{\cX, \cT, \cY} y \cdot p(x) \pi(t \mid x) 
     \cdot p(Y(t) = y \mid x, t) dx dt dy \nonumber \\
      =& \int_{\cX, \cT, \cY} y \cdot p(x) \frac{\pi(t \mid x)}{p(t \mid x)} \nonumber \\
      & \cdot p(Y(t) = y \mid x, t) p(t \mid x) dx dt dy \label{eqn:multiply_by_one} \\
      =& \int_{\cX, \cT, \cY} y \cdot p(x) p(t \mid x) \nonumber \\
      & \cdot p(Y = y \mid x, t) \frac{\pi(t\mid x)}{p(t \mid x)} dx dt dy \label{eqn:ignorability} \\
      =& \int_{\cX, \cT, \cY} y \cdot p(x, t, y) 
       \cdot \frac{\pi(t \mid x)}{p(t \mid x)} dx dt dy \label{eqn:collect} 
\end{align}

Where in Equation~(\ref{eqn:multiply_by_one}) we multiply by one, in Equation~(\ref{eqn:ignorability}) we use the assumption of ignorability to write $p(Y(t) = y \mid X = x, T = t) = p(Y = y \mid X = x, T = t)$ and rearrange terms, and in Equation~(\ref{eqn:collect}) we collect the terms which represent the observed distribution.  For our purposes, it is sufficient to look at the integral in Equation~(\ref{eqn:collect}) to see that it requires the condition that for all $(x, t) \in \cX \times \cT$, the relationship $\pi(T = t \mid X = x) > 0 \implies p(T = t \mid X = x) > 0$ must hold.
\end{proof}

The condition given in Proposition~\ref{thmprop:coverage} is sometimes referred to as the condition of \textit{coverage} \citep[see][Section 5.5]{Sutton2017} in off-policy evaluation.  Rewriting Equation~(\ref{eqn:collect}) as an expectation over the observed distribution, we can see that this leads naturally to the importance sampling \citep{Kahn1955} estimator 
\begin{equation}
\E\left[Y \frac{\pi(T = t \mid X = x)}{p(T = t \mid X = x)}\right] \approx \frac{1}{n} \sum_{i = 1}^{n} y_i \frac{\pi(t_i \mid x_i)}{p(t_i \mid x_i)} \label{eqn:ISexpectation},
\end{equation}
which approximates our desired quantity.  If $\epsilon > p(t | x) > 0$ for some small value of $\epsilon$, then the variance of the importance sampling estimator increases dramatically. This motivates our notion of ``strict'' coverage, that for each value of $x \in \cB^\epsilon(\pi)$, we require that for all actions $t$ such that $\pi(t | x) > 0 $, the condition $p(t | x) > \epsilon$ must hold.

Note that this differs conceptually from the binary treatment case in an important respect:  Since we are not seeking to contrast all treatments, we do not require that $\mu(t | x) > \epsilon, \forall t \in \cT$, but rather just for those treatments which have positive probability of being taken under the target policy.

\paragraph{Algorithmic Details} As described in the main paper, applying OverRule to the policy evaluation setting only requires a single change to the procedure, which is that the set $\hB^\epsilon(\pi)$ is used in place of the set $\hB^\epsilon$ in Equation~\eqref{eqn:binaryclass2} in Section~\ref{sec:overlaprules}.  Nonetheless, we provide an explicit self-contained sketch of the procedure here to avoid any confusion:
\begin{enumerate}
    \item Given a dataset, find an $\alpha$-MV set $\cS^\alpha$ using the approach given in the main paper.
    \item Using this set, learn the conditional probabilities of each possible treatment $t \in \cT$, resulting in estimated propensities $\hat{p}(T = t \mid X = x)$
    \item For each data point in the support set $\cS^\alpha$, assign the label $$\hat{b}_i(\pi) = \prod_{t \in \pi(x_i)} \mathds{1}[\hat{p}(T=t \mid X=x_i) \geq \epsilon],$$ where $\pi(x_i) \coloneqq \{t: \pi(t | x_i) > 0\}$.  The set $\hB^\epsilon(\pi)$ is the collection of data points such that $\hat{b}_i(\pi) = 1$.  Note that we know the target policy $\pi$ that we are evaluating, so we can evaluate $\pi(t | x_i)$ for each data point.
    \item Solve the following Neyman-Pearson-like classification problem, using the techniques discussed in the main paper.  Note that this is identical to solving Equation~\eqref{eqn:binaryclass2} in Section~\ref{sec:overlaprules}, with the substitution of $\hB^\epsilon(\pi)$ for $\hB^\epsilon$:
\begin{align*}\label{eqn:binaryclasspolicyappendix}
    \hcB(\pi) := \argmin_C \;\; \frac{1}{\lvert\hcS \setminus \hB\rvert} \sum_{i\in\hcS \setminus \hB^\epsilon(\pi)} \mathds{1}[x_i \in \cC]  + R(\cC)
    \\
    \text{s.t.} \;\; \sum_{i\in\hcS\cap\hB^\epsilon(\pi)} \mathds{1}[x_i \in \cC] \geq \beta \lvert\hcS\cap\hB^\epsilon(\pi)\rvert~.
\end{align*}
\end{enumerate}

\section{Additional Experimental Results}

As a general note across all experiments:  When estimating support in OverRule, we use $m_R = c \cdot m\cdot d$ uniform reference samples where $c > 0$ is some constant, $m$ is the number of data samples and $d$ their dimension. Continuous features were binarized by deciles unless otherwise specified.
Finally, for propensity-based base estimators, we use the standard threshold $\epsilon=0.1$~\citep{crump2009dealing} throughout.

\subsection{Iris}
For the results given in the paper, we fit OverRule using a $k$-NN base estimator ($k=8$) and DNF Boolean rules for both support and overlap rules, with $\alpha=0.9$ and regularization $\lambda_0=2 \cdot 10^{-2}, \lambda_1=0$ for support rules, a cutoff of $\epsilon=0.1$, and $\beta = 0.9, \lambda_0=10^{-2}, \lambda_1=0$ for overlap rules.

\subsection{Jobs}
For the results given in the main paper, we use the following hyperparameters:
\begin{enumerate}
    \item {\bf Support Rules}: CNF formulation, along with hyperparamters $\alpha = 0.98, \lambda_0 = 10^{-2}, \lambda_1 = 10^{-3}$.  
    \item {\bf Base Estimators}: For CBB we used $\alpha=0.1$, for the logistic regression propensity estimator we used $C = 1$ in \verb|LogisticRegression| in scikit-learn, and other hyperparameters were chosen based on cross-validation:  For $k$-NN, we selected $k\in \{2, 4, \ldots, 20\}$ based on held-out accuracy in predicting group membership and used $1/k$ as threshold. For OSVM, we use a Gaussian RBF-kernel with bandwidth $\gamma \in [10^{-2}, 10^{2}]$, selected based on the held-out likelihood of kernel density estimation.
    \item {\bf Overlap Rules}: We use a DNF formulation with $\beta = 0.9$ and select $\lambda_0 \in [10^{-4}, 10^{-1}]$ and $\lambda_1 \in [10^{-4}, 10^{-2}]$.  Within each class of base estimators, we choose these parameters based on average \textit{training} performance over 5-fold CV, choosing the setting in each class that achieves a balanced accuracy (with respect to the base-estimator overlap labels) within 1\% of the best performing model in the class, while minimizing the number of rules.  
\end{enumerate}

Note that the reported results are using the held-out portions of each 5-fold CV run, and using the ground-truth overlap labels, which are at no point used during the hyperparameter tuning process. This reflects a real-world scenario where ground-truth is unknown and only the base-estimator derived labels are given.  The reported rules in the figure were selected from one of the five cross-validation runs for the same hyperparameter setting chosen using the above procedure.  In Figure~\ref{fig:jobs} we see the correlation between held-out balanced accuracy for the rule set w.r.t. the experimental label, and the balanced accuracy for the rule set in approximating the base estimator. Note that AUC is equal to balanced accuracy for binary predictions.

\begin{figure}[ht!]
    \centering
    \begin{subfigure}{.98\columnwidth}
        \centering
       \includegraphics[width=.95\columnwidth]{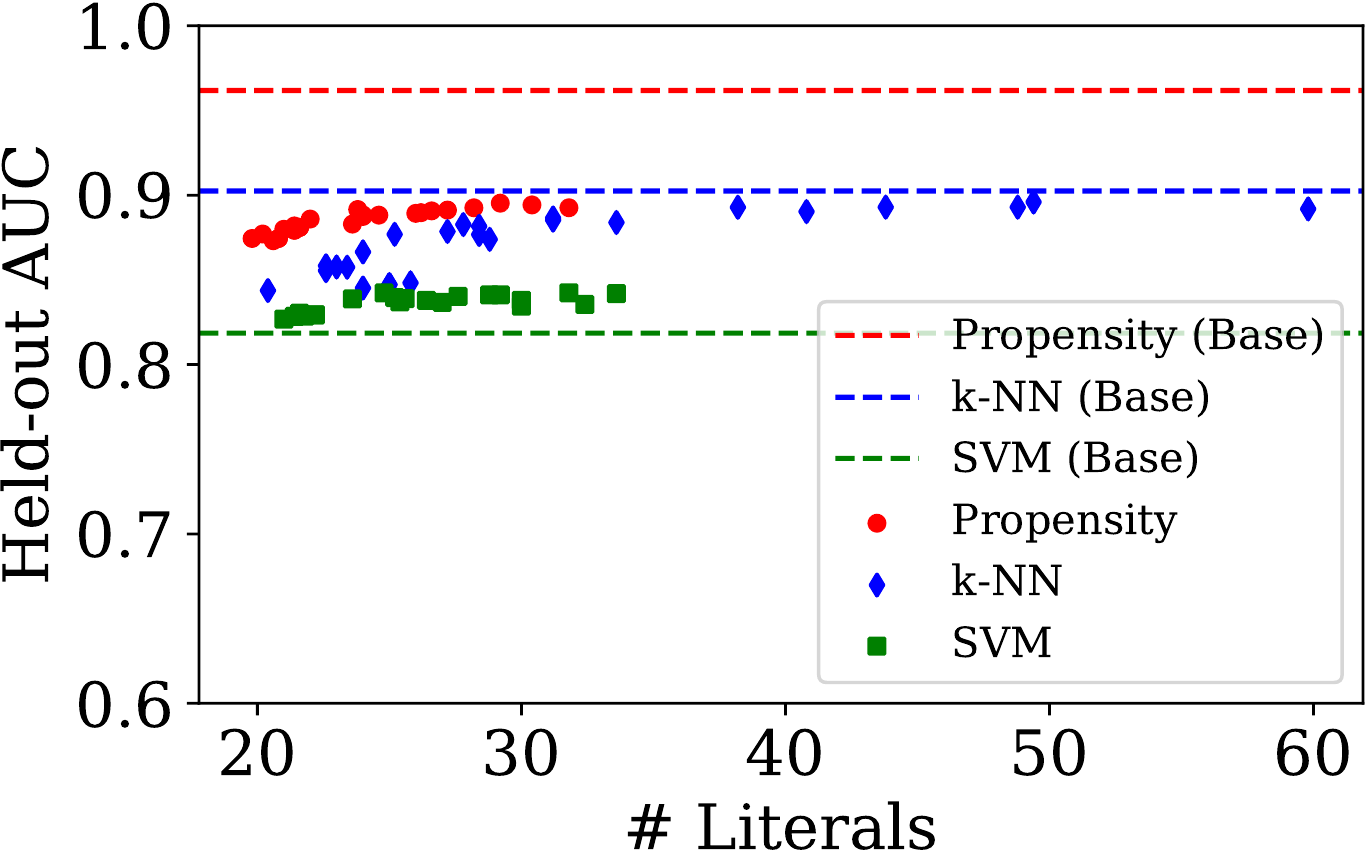}
       \label{fig:jobs_lit}
    \end{subfigure}
    
    \begin{subfigure}{.98\columnwidth}
        \centering
        \includegraphics[width=.95\columnwidth]{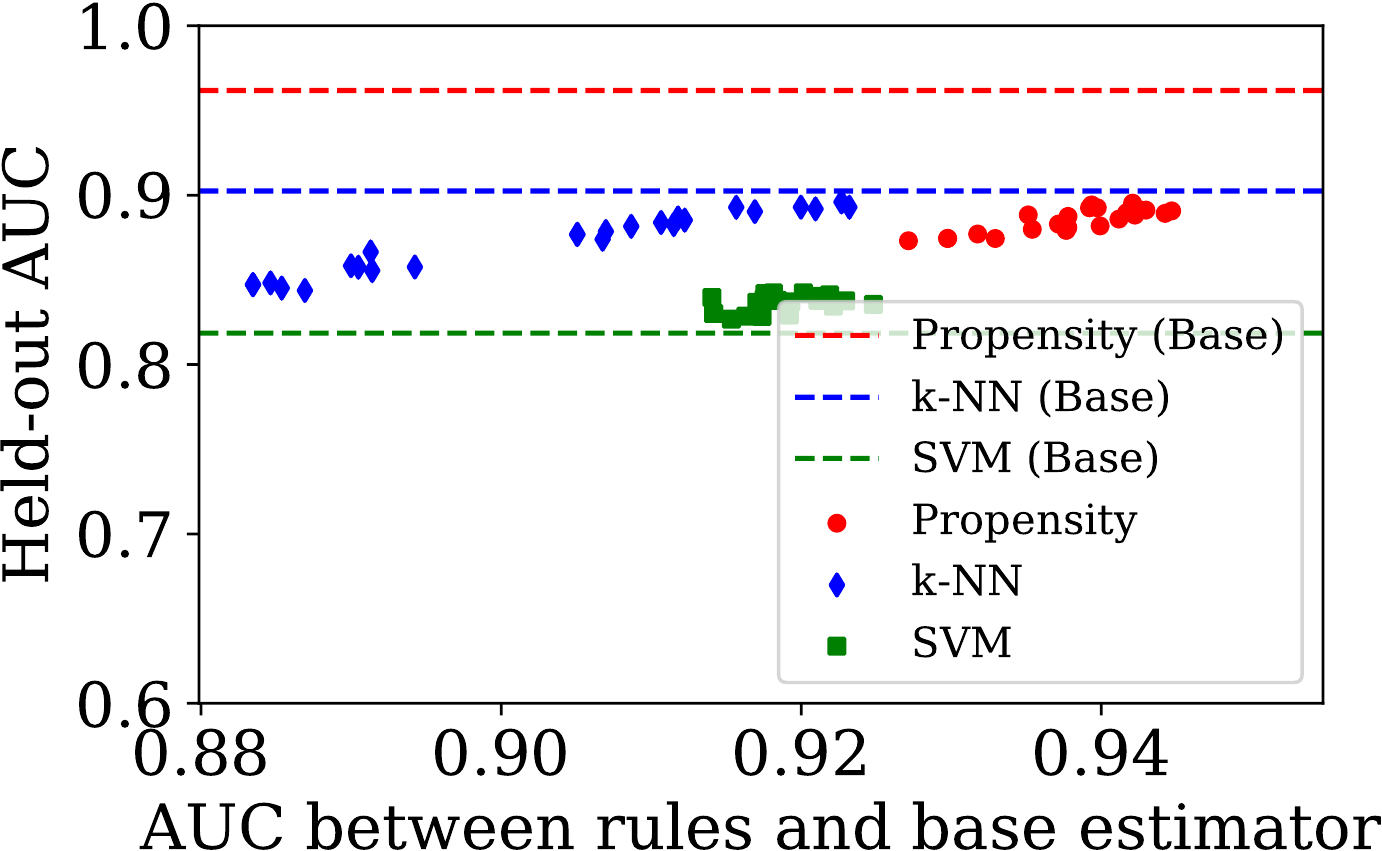}
        \label{fig:jobs_base}
    \end{subfigure}
   \caption{Results from the Jobs datasets for OverRule approximations of different base estimators, sweeping $\lambda_0, \lambda_1$. AUC (i.e., balanced accuracy) is measured with respect to the experimental indicator.  The dotted line `Propensity (base)' refers to the logistic regression base estimator, `k-NN (base)` refers to the k-NN base estimator, and `SVM (base)` refers to the one-class SVM.  The colored points refer to performance of OverRule using the respective base estimator, for different values of $\lambda_0, \lambda_1$}
    \label{fig:jobs}
\end{figure}

\subsection{Opioids}
\label{sub:appendix_opioids}
For the results in the main paper, we fit an OverRule model (OR) to a random forest base estimator with $\beta=0.8$ for $\cB$ and $\alpha=0.9$ for $\cS$ picked a priori. The hyperparameter $\lambda_0$ was set to $\lambda_0=$\num{1e-3} for $\cB$, and $\lambda_0=$\num{1e-5} for $\cS$, and $\lambda_1=0$ for both.

For a full table of covariate statistics for the Opioids dataset, see Table~\ref{tbl:covariates_supp}. For a illustration of the rules learned by OverRule to describe the complement of the overlap set, see Figure~\ref{fig:opioid_rules_comp}.

\begin{figure*}
    \centering
    \includegraphics[width=\textwidth]{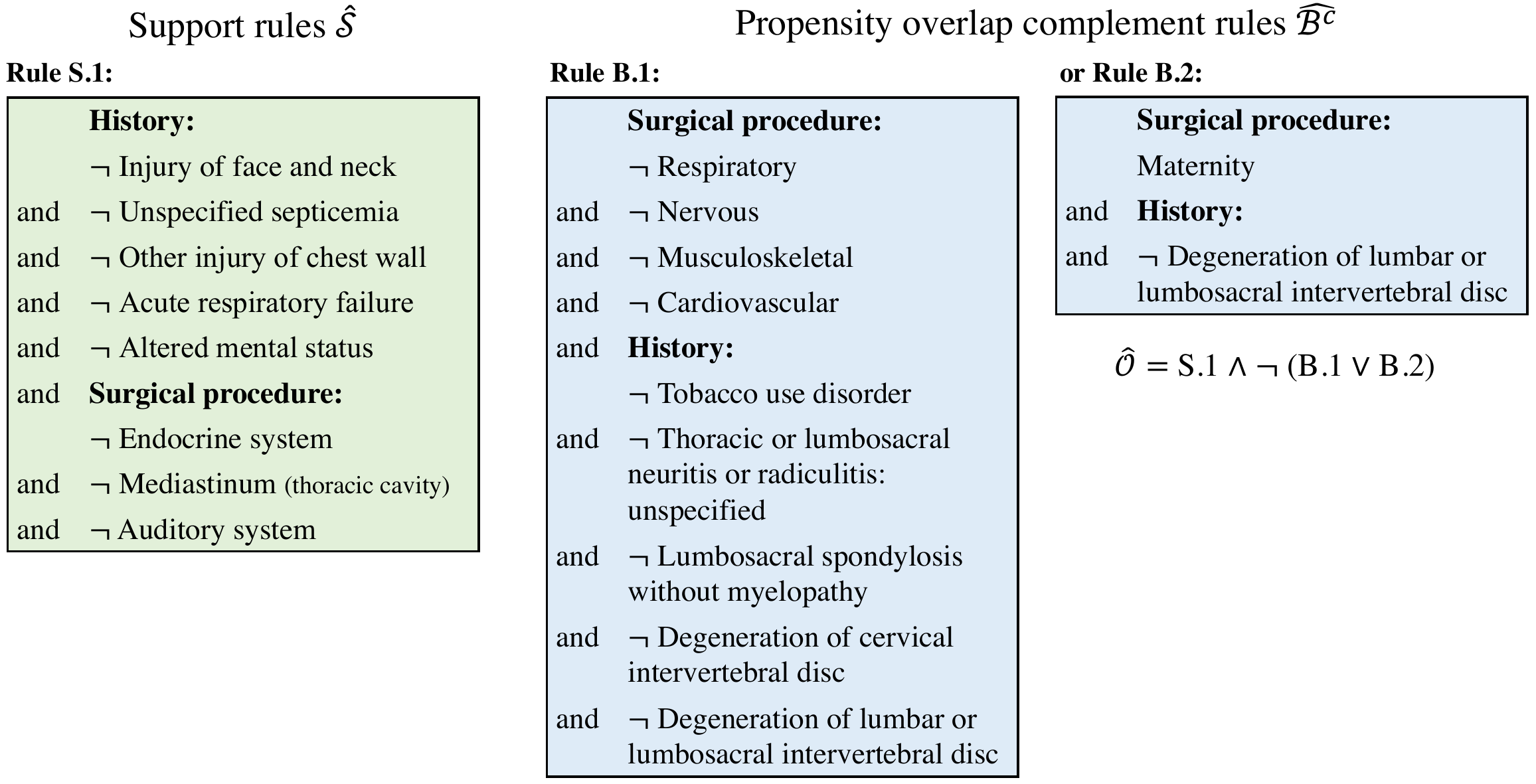}
    \caption{OverRule description of the \emph{complement} of the overlap between post-surgical patients with higher and lower opioid prescriptions. If the support rule (left) applies and \emph{neither} propensity overlap rule (right) applies, a patient is consider to be in the overlap set. $\lnot$ indicates a negation. The rules cover 36\% of patients with balanced accuracy 0.92 w.r.t. the base estimator (random forest). Procedures are not mutually exclusive.\label{fig:opioid_rules_comp}}
\end{figure*}

{\bf Supplemental Rules:}  We learned an additional set of rules, motivated by our experiments in Section~\ref{sec:opioidexp}, where we noted that the support rules did not capture certain combinations of surgery types or conditions that should be rare or non-existent.  This motivated the empirical investigation in Section~\ref{sub:hyperparameter_support}, and this vignette represents the result of re-running our procedure with this goal in mind.

For support rules, we followed the recommendations laid out in Section~\ref{sub:hyperparameter_support}, choosing to use a CNF formulation with $\alpha=0.98, \lambda_0 = 0, \lambda_1 = 0.01$.  Continuous features were binarized using deciles.  For our base estimator, we used a random forest classifier with 100 trees and 20 minimum samples per leaf, and we used $\epsilon=0.1$ as our cutoff.  For the overlap rules, we searched over the following grid of hyperparameters, with the goal of maximizing balanced accuracy with respect to the overlap labels on a validation set: $\beta \in \{0.8, 0.9, 0.95\}$ and then a set where $\lambda_0 = 0$ and $\lambda_1 \in \{10^{-3}, 2 \cdot 10^{-3}, 10^{-2}\}$, and a set where $\lambda_1 = 0$ and $\lambda_0 \in \{10^{-3}, 2 \cdot 10^{-3}, 10^{-2}\}$.  The selected hyperparameters were $\beta = 0.95, \lambda_0 = 0, \lambda_1 = 10^{-3}$.  The support rules cover 98.5\% of the test samples, and the overlap rules achieved a balanced accuracy of 0.96 on a held-out test set (with respect to the overlap labels) and covered 36\% of the test samples. The chosen ruleset is given in given in Figures~\ref{fig:supp_newhp_opioids_support}-\ref{fig:supp_newhp_opioids_overlap}. 

We note that the resulting support rules, in line with the findings in Section~\ref{sub:hyperparameter_support}, include a large number of rules that exclude zero training data points, by identifying rare interactions of features.  For instance, the rules identify that there are \textit{no men in our dataset who have maternity surgery}, an intuitive exclusion.

We shared this rule set with one of the participants of the original user study, who made the following observations:  First, the support rules in Figure~\ref{fig:supp_newhp_opioids_support} generally made sense as excluding combinations that are intuitively absent from the data (e.g., men w/maternity surgery) or that are just combinations of features that are themselves rare.  Regarding the overlap rules in Figure~\ref{fig:supp_newhp_opioids_overlap}, they observed that B.1 and B.2 were consistent with clinical intuition, where B.2 likely serves to exclude C-section patients with epidurals.  B.3 and B.4 were intuitive with the exception of the negations, e.g., it is unclear what the role of abdominal pain is in B.3, although it could be correlated with generalized pain syndromes. B.5-B.7 correspond to individuals with lower back pain (Lumbago) and neck pain (Cervicalgia) which are intuitive indicates for higher doses of opioids. B.8 corresponds to plastic surgery, and the broad category of respiratory surgery in B.9 could correspond to thoracic surgery, one of the main surgical categories associated with opioid misuse.  B.10-B.12 relate to back pain, which is associated with higher opioid dosages.

\begin{figure*}[h]
    \centering
    \includegraphics[width=.90\textwidth]{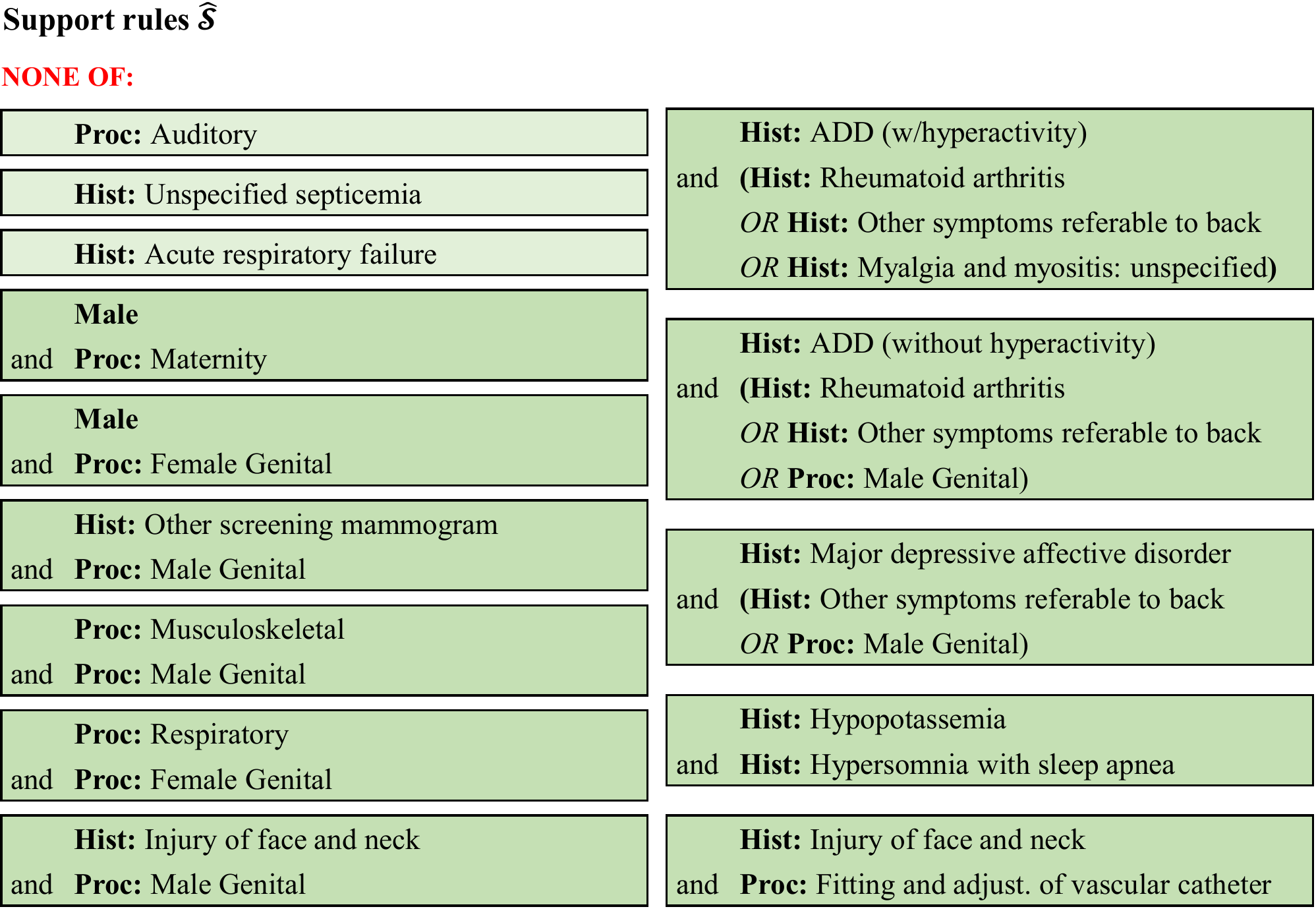}%
    \caption{Support Rules using CNF formulation for the Opioids task. \textbf{Proc} indicates a procedure, and \textbf{Hist} indicates a history of a condition.  A sample is considered in the support set if NONE of the above rules apply.  Note that rules are negated for simplicity of presentation, as ``AND NOT (X AND Y)'' is equivalent to ``AND (NOT X OR NOT Y)'', and in some cases several rules are combined for simplicity of presentation (e.g., those related to Attention Deficit Disorder).  Dark green rules are highlighted to indicate that they cover <4 training samples (and in many cases zero training samples) in line with our findings in Section~\ref{sub:hyperparameter_support} for this setting of hyperparameters.}
    \label{fig:supp_newhp_opioids_support}
\end{figure*}
\begin{figure*}[h]
    \centering
    \includegraphics[width=.90\textwidth]{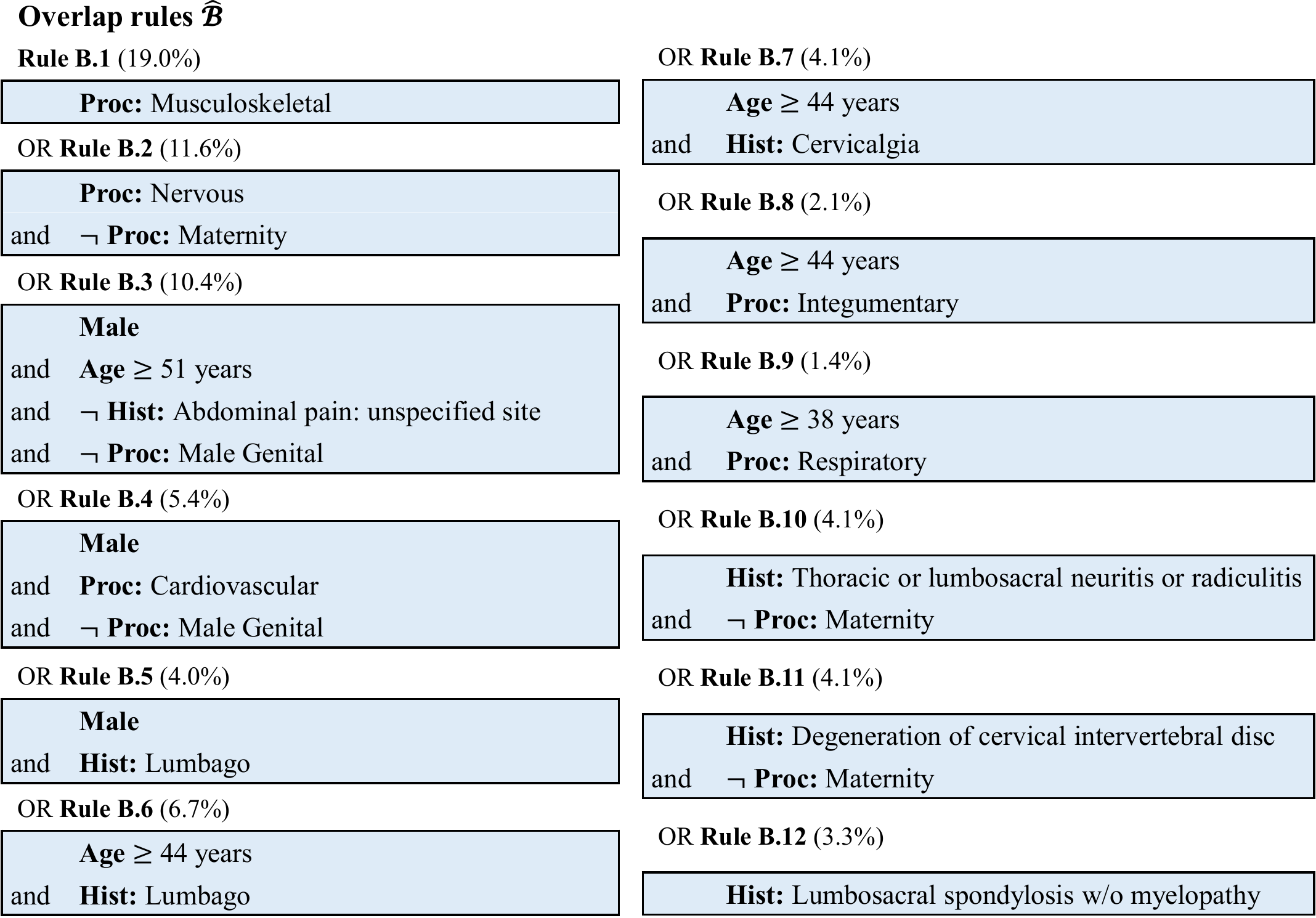}%
    \caption{Overlap rules, where the percentage next to each rule indicates the percentage of the dataset that is covered by that rule.  Collectively, these rules cover 36\% of the held-out datapoints.}
    \label{fig:supp_newhp_opioids_overlap}
\end{figure*}%

\subsection{Observational Study: Policy Evaluation of Antibiotic Prescription Guidelines}
\label{sec:utiappendix}

Antibiotic resistance is a growing problem in the treatment of urinary tract infections (UTI) \citep{Sanchez2016}, a common infection for which more than 1.6 million prescriptions are given annually in the United States \citep{Shapiro2013}.  With this in mind, we are interested in the following clinical problem:  When a patient presents with a UTI, the physician needs to choose between a range of antibiotics, with the dual goals of (a) treating the infection, and (b) minimizing the use of broad-spectrum antibiotics, which are more likely to select for drug-resistant strains of bacteria.

In this context, we might be interested in evaluating a range of potential treatment policies.  For our purposes, we will use a pre-defined policy: The clinical guidelines published by the Infectious Disease Society of America (IDSA) for treatment of uncomplicated UTIs in female patients \citep{Gupta2011}. Using the policy evaluation formulation of $\cB^\epsilon(\pi)$, we will apply OverRule to a conservative interpretation of the IDSA guidelines, using data curated from the Electronic Medical Record (EMR) of two academic medical centers.

The official guidelines discuss the importance of patient and population level risk factors in predicting resistance, and include some factors that we do not observe in our data (such as drug allergies). In order to characterize the guideline explicitly as a policy that we can evaluate in our dataset, we used the following interpretation:
\begin{itemize}
    \item Choose the first-line agent, either Nitrofurantoin (NIT) or Trimethoprim/Sulfamethoxazole (SXT), to which the patient did not have previous antibiotic exposure or resistance in the prior 90 days.  Additionally, if local rates of resistance to SXT are $\geq 20\%$ in the prior 30-90 days, then avoid prescription of SXT.
    \item If neither of the first-line agents are indicated, then prescribe Ciprofloxacin (CIP), a second-line agent.
\end{itemize}

\paragraph{Experimental details} From our data set, we selected all patients from 2007--2017 which had a UTI, and were prescribed one of the four most common antibiotics: NIT, SXT, CIP, or Levofloxacin (LVX).
Features include demographics (race, gender, age, and veteran status), comorbidities observed in the past 90 days, information about previous infections (organism, antibiotics given, and resistance profile), hospital ward (inpatient, outpatient, ER, and ICU), and indicators for pregnancy and nursing home residence in the past 90 days. The local rates of resistance (for each hospital ward) are given over the past 30--90 days, and used at the patient level as a feature, as well as an input to the decision of the guidelines.  

We preprocess our data first, removing any binary feature with a prevalence of less than 0.1\%, and any associated subject:  This results in the removal of 48 binary features with less than 0.1\% prevalence and 888 corresponding subjects. This leaves a total of 156 (150 binary, 6 continuous) features and 64593 subjects.  Detail on all remaining features are given in Table~\ref{tbl:covariates_uti}.  For the purposes of running our algorithm, we convert all continuous variables into binary variables by using indicator functions for deciles.

We then characterize the support set $\cS^\alpha$ as described in the main paper, using a DNF formulation, along with $\alpha = 0.95, \lambda_0 = 0.01, \lambda_1 = 0$.  Using the data points which fall into the support set, we then estimate the propensity $p(t | x)$ of prescribing each of the four drugs using a random forest classifier, with hyperparameter selection done using 5-fold cross-validation on 80\% of the remaining cohort used as a training set, over the following parameter grid: Number of estimators $\in [100, 500]$, Minimum samples per leaf (as fraction of total) $\in [0.005, 0.01, 0.02]$.  The resulting calibration curves for each antibiotic are given in Figure~\ref{fig:comp-calibration}, using the remaining held-out 20\% of the data.  Using these propensity scores, we apply the procedure described in Section~\ref{sec:policyevalappendix} to estimate the region of strict coverage, $\hcB^\epsilon(\pi)$ using Boolean rules, and the resulting rules are given in Figure~\ref{fig:comp-rules}. For this stage, we used a DNF formulation and hyperparameters of $\beta = 0.9, \lambda_0 = 0.03, \lambda_1 = 0$.

\begin{figure}[ht!]
    \centering
    \includegraphics[width=\columnwidth]{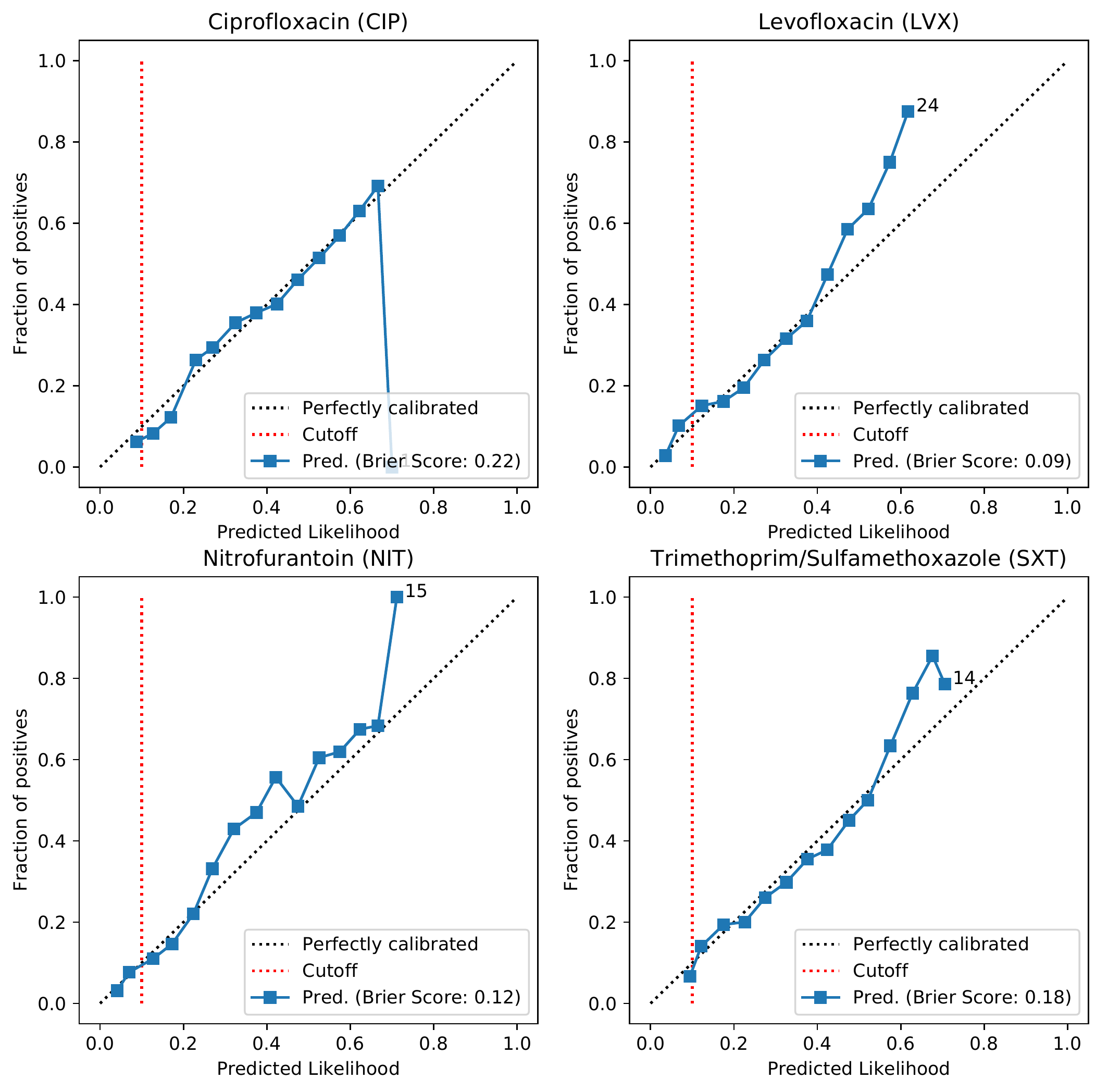}
    \caption{Calibration curves for each antibiotic, using 20 evenly spaced bins in the range $[0, 1]$.  Numbers indicate the number of samples, and are given when when the number of samples in a bin is less than 0.5\% of the total.  The cutoff is a reminder that $\epsilon = 0.1$ in this experiment:  For any subject with covariates $x$, the propensity must be above this cutoff for every treatment under the target policy (i.e., for all $t$ such that $\pi(t | x) > 0$) for them to be included in the coverage region.} 
    \label{fig:comp-calibration}
\end{figure}

\begin{figure*}[ht!]
    \centering
    \includegraphics[width=\textwidth]{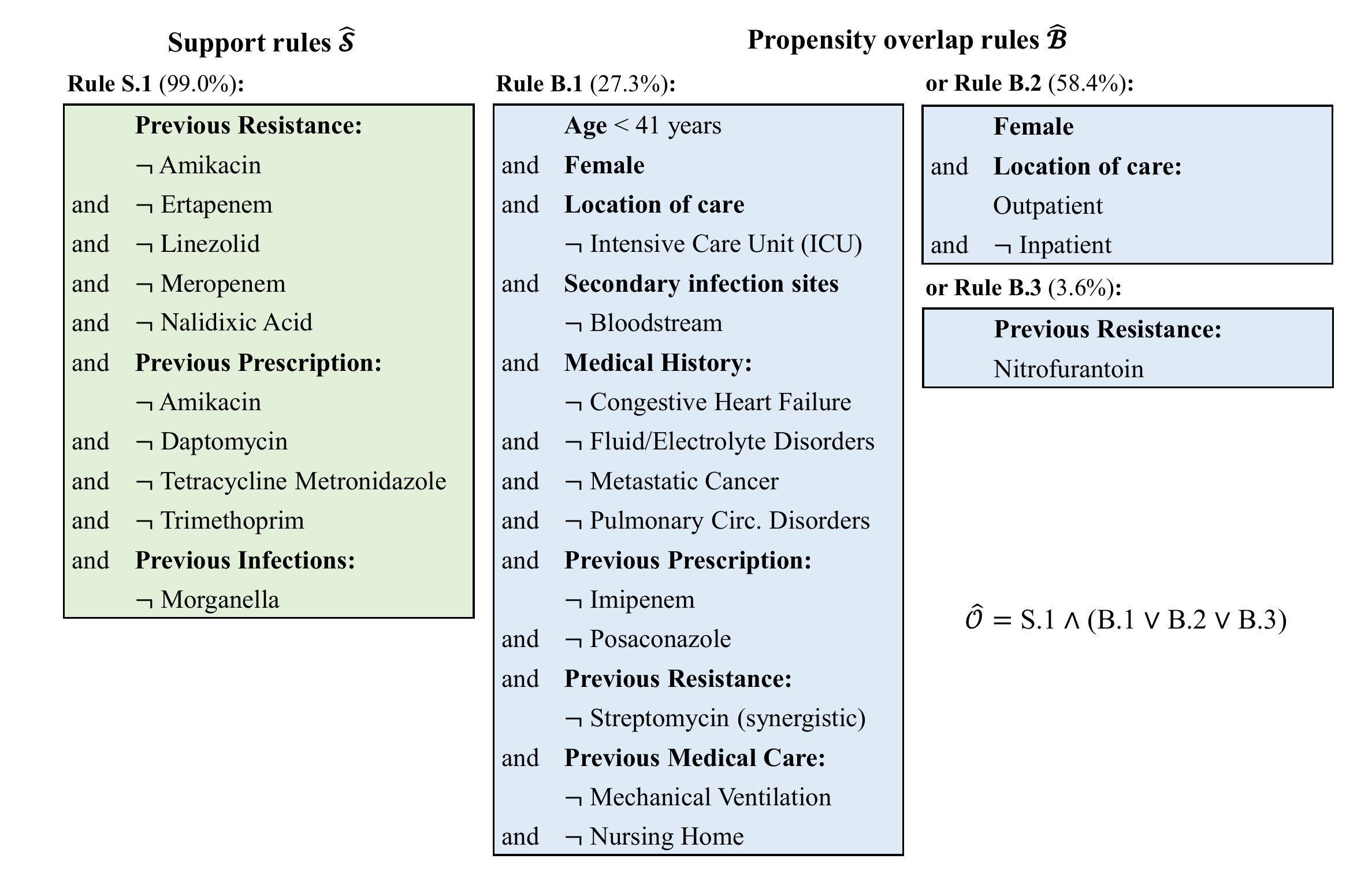}
    \caption{OverRule description of the coverage region for policy evaluation of the clinical guidelines.  Beside each rule we give the percentage of subjects that are covered by the rule in the test set.  Overall, the rules for $\hat{B}$ cover 65.4\% of the data points in the support region (compared to the 71\% of points labelled by our base estimator), and they have an balanced accuracy of 0.96 versus the base estimator.}
    \label{fig:comp-rules}
\end{figure*}

\paragraph{Clinical Validity / Interpretation} Towards understanding the clinical validity of these rules, we interviewed a clinician who specialises in infectious diseases\fxnote{confirm appropriate title w/SJ}.  First, we asked them, based on the available features, which they would expect to differentiate between subjects for whom the policy is or is not followed.  They noted that the guidelines are designed for uncomplicated cases:  In particular, patients who have a Foley catheter (a catheter used to drain urine from the bladder) are not covered under these guidelines, because infections in these patients tend to be more complex (e.g., the infection could have been introduced by the catheter itself).  The use of the Foley catheter is common during intensive care (e.g., in the ICU), so complex hospitalized patients are less likely to be treated according to the guidelines.

With that in mind, they reviewed the available features and noted the following: \begin{enumerate*}[label=(\roman*)] 
    \item While UTIs are common for women, they are rare for men;  Men with UTIs tend to be more complicated cases, because it is indicative of deeper abnormalities.  Similarly, pregnant women are excluded from the guidelines.
    \item Of the comorbidities given, none of them should directly disqualify patients from the guidelines, except potentially for complicated diabetes.  
    \item Prior organisms / resistance / prescriptions should not directly disqualify patients from the guidelines, though they will influence the type of antibiotic given.  In particular, if a patient has had previous resistance to an antibiotic, they are unlikely to be prescribed it again.
    \item The previous procedures given (with the exception of surgery) are associated with ICU patients.  For instance, mechanical ventilation and parenteral nutrition are exclusive to the ICU, and those patients likely have a Foley catheter as well.  Surgery is too broad of a category to draw any conclusions.
    \item In terms of locations besides the ICU, patients who are admitted to the hospital and who are on intravenous (IV) antibiotics already will be treated differently.  The guidelines are focused on oral antibiotics, whereas if an IV already exists, additional IV antibiotics are likely to be given instead.
\end{enumerate*}

Having discussed these points first, we then showed them the rules learned by the OverRule algorithm, and asked for their interpretation, as well as for any critiques of the rules based on their clinical knowledge.  Their reaction to each of the rules was as follows: 
\begin{itemize}
    \item \textbf{Rule B.1}: This appears to correspond to a relatively straightforward young inpatient female (given that Rule B.2 covers all outpatient females).  In particular, it rules out ICU patients directly, as well as those with recent mechanical ventilation, which would indicate a recent ICU stay.  It also rules out patients with current bloodstream infections, and those who had previously been tested for (and found to be) resistance to Streptomycin (synergistic):  This is only tested for in the context of bloodstream infections by enterococcus, and would be an indicator of previous bloodstream infections.  Imipenem is an IV antibiotic only given in inpatient settings, and posaconazole is an antifungal used in bone marrow transplant patients.  Patients who are both young and in a nursing home tend to be more complex, e.g., they may be paralysed or otherwise unable to perform activities independently.  Finally, the excluded comorbidities are less intuitive, because some of them (e.g., congestive heart failure) manifest with a range of severity:  For patients with controlled congestive heart failure, this is not a contraindication for following the guidelines, but if they are fully decompensated, then they would likely be on a Foley catheter.
    \item \textbf{Rule B.2}: This concisely describes the most common manifestation of UTI and the set of patients who are most likely to be treated according to the guidelines\footnote{Note that outpatient and ``not inpatient'' can appear in the same rule without being redundant, because multiple specimens collected on the same day for the same patient are collapsed into a single subject.}.
    \item \textbf{Rule B.3}: The conjecture is that this represents patients who have had an uncomplicated UTI in the past, since patients are usually tested for the antibiotics under consideration by a physician, and since nitrofurantoin is one of the first-line treatments for uncomplicated UTIs.
\end{itemize}

From a quantitative perspective, we compared the learned region with an explicitly constructed cohort of patients whose inclusion criteria were explicitly designed to make them eligible for application of the IDSA guidelines.  In particular, we defined this cohort as including non-pregnant women between the ages of 18 to 55 years of age with no record of genitourinary surgery or instrumentation, immunosuppression, indwelling catheters, or neurologic dysfunction in the preceding 90 days.  There were 14k of these patients, 21\% of the total.

In relationship to this conservative subset, the learned region (covering 42k patients, 64\% of total) covers 96\% of the explicitly constructed cohort, while also demonstrating that a broader set of patients are treated according to these guidelines in practice.

\section{Theoretical Results on Regularized Minimum-Volume Boolean Rules}
\label{sec:MVtheory}

\subsection{Bounds on minimum volume}
\label{sec:volBound}

In this subsection, we derive lower bounds on the volume of optimal DNF Boolean rules in problem \eqref{eq:MVobj}. 

First we obtain an expression for the normalized volume of a clause in a DNF (we use the terms clause and conjunction interchangeably in the case of a DNF). We express the domain $\cX$ as the Cartesian product $\cX_1 \times \dots \times \cX_d$. A DNF rule with $K$ clauses $a_k$ is written as 
\begin{equation}\label{eqn:DNF}
r(x) = \bigvee_{k=1}^K a_k(x) = \bigvee_{k=1}^K \bigwedge_{j\in\cJ_k} \left(x_j \in \cS_{jk}\right),
\end{equation}
where $\cJ_k$ is the set of covariates participating in clause $k$, and each $x_j \in \cS_{jk} \subseteq \cX_j$ is a subset membership condition on an individual covariate. Examples of such conditions are (Age $\geq 30$) for a continuous-valued covariate and (Sex = Female) for a discrete-valued one. For $j \notin \cJ_k$, it is understood that $x_j \in \cX_j$, i.e.~there is no restriction on $x_j$. The volume of clause $a_k$ is then given by the product 
\[
V(a_k) = \prod_{j\in\cJ_k} \lvert\cS_{jk}\rvert \prod_{j\notin\cJ_k} \lvert\cX_j\rvert, 
\] 
where $\lvert\cS_{jk}\rvert$ is the length of subset $\cS_{jk}$ for a continuous covariate $j$ or the cardinality of $\cS_{jk}$ for a discrete covariate, and similarly for $\lvert\cX_j\rvert$. 
Likewise, the volume of $\cX$ is $\prod_{j=1}^d \lvert\cX_j\rvert$,
and the normalized volume of $a_k$ is therefore 
\begin{equation}\label{eqn:volConj}
\bar{V}(a_k) = \prod_{j\in\cJ_k} f_{jk}, \quad f_{jk} = \frac{\lvert\cS_{jk}\rvert}{\lvert\cX_j\rvert} \in [0,1].
\end{equation}

We define $p_k = \lvert\cJ_k\rvert$ to be the \emph{degree} of conjunction $k$.
\begin{thmprop}
Assume that the regularization $R(r)$ follows \eqref{eqn:R}. Then in any optimal solution to \eqref{eq:MVobj}, all clauses $a_k$ of degree $p_k$ have normalized volume satisfying $\bar{V}(a_k)^{(p_k-1)/p_k} - \bar{V}(a_k) \geq \lambda_1$.
\label{prop:volConj}
\end{thmprop}
\begin{proof}
Suppose that rule $r$ with corresponding set $\cC$ is an optimal solution to \eqref{eq:MVobj}. Recalling the expansion in \eqref{eqn:DNF}, we consider modifications to $r$ in which one condition $(x_j \in \cS_{jk})$ is removed from a clause $a_k$. The modified rule satisfies the mass constraint $P(\cC) \geq \alpha$ because it covers at least those points covered by $r$. From \eqref{eqn:volConj}, the increase in volume is at most $\bar{V}(a_k)((1/f_{jk}) - 1)$, with equality if none of the additional volume is already covered by another clause in $r$, while the complexity penalty decreases by $\lambda_1$. The change in objective value is thus bounded from above by 
\[
\bar{V}(a_k)\left(\frac{1}{f_{jk}} - 1\right) - \lambda_1.
\]
This upper bound must be non-negative as otherwise $r$ is not optimal. In particular, for $f_{jk} = \max_{j'\in\cJ_k} f_{j'k}$ and all $k$ we have
\[
\bar{V}(a_k)\left(\frac{1}{\max_{j\in\cJ_k} f_{jk}} - 1\right) \geq \lambda_1.
\]
Since \eqref{eqn:volConj} implies that $\max_{j\in\cJ_k} f_{jk} \geq \bar{V}(a_k)^{1/p_k}$, the desired result follows.
\end{proof}

For $p > 1$, the function $\bar{V}^{(p-1)/p} - \bar{V}$ is positive and concave on $(0, 1)$ with roots at $0$ and $1$. For $\lambda_1 > 0$, the equation $\bar{V}^{(p-1)/p} - \bar{V} = \lambda_1$ therefore has either two roots, $0 < \bar{V}_L < \bar{V}_U < 1$, which define an interval where the inequality $\bar{V}^{(p-1)/p} - \bar{V} \geq \lambda_1$ is satisfied, or no roots if $\lambda_1$ is too large. We are interested primarily in the root $\bar{V}_L$ as a lower bound on volume. While $\bar{V}_L$ is not available in closed form for $p > 2$, the following corollary gives a simple expression that is a lower bound on $\bar{V}_L$.
\begin{thmcol}\label{cor:volConj}
Under the assumption in Proposition~\ref{prop:volConj}, in any optimal solution to \eqref{eq:MVobj}, all clauses $a_k$ of degree $p_k > 1$ have normalized volumes of at least $\lambda_1^{p_k/(p_k-1)}$. 
\end{thmcol}
\begin{proof}
Proposition~\ref{prop:volConj} implies $\bar{V}(a_k)^{(p_k-1)/p_k} \geq \lambda_1$ after dropping $-\bar{V}(a_k)$ from the left-hand side.
\end{proof}

Lastly, since the volume of a DNF rule is at least that of any of its clauses, we have the following.
\begin{thmcol}
Under the assumption in Proposition~\ref{prop:volConj}, any optimal solution to \eqref{eq:MVobj} has normalized volume of at least $\lambda_1^{p_{\max}/(p_{\max}-1)}$, where $p_{\max} = \max_k p_k$ is the largest degree of its clauses. 
\end{thmcol}

\subsection{Bounds on the number of candidate DNF rules}

The results in the previous subsection are necessary conditions of optimality for problem \eqref{eq:MVobj}. The implication is that in searching for optimal solutions to \eqref{eq:MVobj}, we may restrict the class $\scrC$ of DNF rules considered to those satisfying these necessary conditions. In this subsection, we develop the consequences of this restriction, culminating in a bound on $\lvert\scrC\rvert$, the number of candidate DNF rules (Lemma~\ref{lem:classSize}).

For simplicity, we assume in the following that all variables $X_j$ are binary-valued. An extension to non-binary categorical variables and continuous variables (discretized using interval conditions $l_j \leq x_j \leq u_j$) is likely possible with the additional complications of accounting for the cardinalities of categorical variables and bounding the fractions $f_{jk}$ associated with continuous variables.

First, the simplified lower bound on volume in Corollary~\ref{cor:volConj} implies an upper bound on conjunction degree.
\begin{thmlem}\label{lem:maxDeg}
Assume that the regularization $R(r)$ follows \eqref{eqn:R} and that all variables are binary. Then in any optimal solution to \eqref{eq:MVobj}, the maximum degree of a conjunction is $p_{\max} := 1 + \lfloor \log_2(1/\lambda_1) \rfloor$.
\end{thmlem}
\begin{proof}
The normalized volume of a conjunction of degree $p_k$ is $2^{-p_k}$. Corollary~\ref{cor:volConj} then requires 
\[
2^{-p_k} \geq \lambda_1^{p_k/(p_k-1)}.
\]
Taking logarithms and rearranging, we obtain
\begin{align*}
    -1 &\geq \frac{1}{p_k-1} \log_2 \lambda_1,\\
    p_k &\leq 1 + \log_2(1 / \lambda_1).
\end{align*}
The right-hand side can be rounded down since $p_k$ is integer.
\end{proof}

Given Lemma~\ref{lem:maxDeg}, we may enumerate DNF rules satisfying the lemma according to the numbers of clauses of degree $p = 1,\dots,p_{\max}$ that they possess. Denote by $K_p$ the number of clauses of degree $p$ and call $\bK = (K_1, \dots, K_{p_{\max}})$ the \emph{signature} of a DNF rule. The signatures of optimal DNF rules obey the following constraint.
\begin{thmlem}\label{lem:sigCons}
Under the assumptions of Lemma~\ref{lem:maxDeg}, the signature $\bK = (K_1, \dots, K_{p_{\max}})$ of an optimal solution to \eqref{eq:MVobj} must satisfy 
\begin{equation}\label{eqn:sigCons}
\sum_{p=1}^{p_{\max}} K_p (\lambda_0 + p\lambda_1) < 1.
\end{equation}
\end{thmlem}
\begin{proof}
From \eqref{eqn:R}, the complexity penalty of a solution with $K_p$ clauses of degree $p$, $p = 1,\dots,p_{\max}$ is given by the left-hand side of \eqref{eqn:sigCons}. For a solution to be optimal, it must have lower cost than the trivial ``all true'' rule, which has a normalized volume of $1$ and complexity penalty of $0$. In particular, the complexity penalty must be less than $1$.
\end{proof}

Let $\Delta$ denote the set of signatures that satisfy \eqref{eqn:sigCons}, and for $\bK \in \Delta$, let $\scrC(\bK)$ be the set of DNF rules with signature $\bK$. The number of DNF rules satisfying the necessary conditions of optimality in Lemmas~\ref{lem:maxDeg} and \ref{lem:sigCons} can be bounded as follows:
\begin{equation}\label{eqn:classSize}
    \lvert\scrC\rvert = \sum_{\bK\in\Delta} \lvert\scrC(\bK)\rvert
    \leq \lvert\Delta\rvert \max_{\bK\in\Delta} \; \lvert\scrC(\bK)\rvert.
\end{equation}
The next two lemmas provide upper bounds on the two right-hand side factors in \eqref{eqn:classSize}.

\begin{thmlem}\label{lem:simplex}
The number of signatures satisfying \eqref{eqn:sigCons} is bounded as 
\[
\lvert\Delta\rvert \leq 2 \left(\frac{1}{\lambda_1}\right)^{p_{\max}}.
\]
\end{thmlem}
\begin{proof}
For simplicity, we consider a superset $\Delta_0 \supseteq \Delta$ obtained by dropping $\lambda_0$ from \eqref{eqn:sigCons}, i.e.
\begin{equation}\label{eqn:sigCons0}
\sum_{p=1}^{p_{\max}} p \lambda_1 K_p \leq 1.
\end{equation}
Condition \eqref{eqn:sigCons0} together with the implicit non-negativity constraints $K_p \geq 0$, $p = 1,\dots,p_{\max}$ define a simplex in $p_{\max}$ dimensions. Bounding the number of signatures in $\Delta_0$ is thus equivalent to bounding the number of non-negative integer points in this simplex. This problem has been studied extensively by mathematicians. Applying e.g.~\cite[eq.~(1.5)]{yau2006simplex}, we have 
\begin{align*}
\lvert\Delta_0\rvert &\leq \frac{1}{p_{\max}!} \prod_{p=1}^{p_{\max}} \frac{1}{p\lambda_1} \left(1 + \sum_{p=1}^{p_{\max}} p\lambda_1 \right)^{p_{\max}}\\
&= \frac{1}{\bigl(p_{\max}!\bigr)^2} \left(\frac{1}{\lambda_1}\right)^{p_{\max}} \left(1 + \frac{p_{\max} (p_{\max}+1) \lambda_1}{2} \right)^{p_{\max}}\\
&\leq \left(\frac{1}{\lambda_1}\right)^{p_{\max}} \underbrace{ \frac{\left(1 + p_{\max} (p_{\max}+1) 2^{-p_{\max}} \right)^{p_{\max}}}{\bigl(p_{\max}!\bigr)^2} }_{F(p_{\max})},
\end{align*}
where the last inequality is obtained by using the definition of $p_{\max}$ in Lemma~\ref{lem:maxDeg} to bound $\lambda_1 / 2 \leq 2^{-p_{\max}}$. 

To complete the proof, we bound the function $F(p_{\max})$ from above. The numerator of $F(p_{\max})$ converges to $1$ as $p_{\max} \to \infty$, as seen by taking its logarithm and bounding it:
\begin{multline*}
p_{\max} \log \left(1 + p_{\max} (p_{\max}+1) 2^{-p_{\max}} \right)\\ 
\leq p_{\max}^2 (p_{\max}+1) 2^{-p_{\max}} \to 0 \quad \text{as } p_{\max} \to \infty.
\end{multline*}
Thus $F(p_{\max})$ decreases to zero as $p_{\max}$ increases. Numerical evaluation shows that $F(p_{\max})$ attains a maximum value of $2$ at $p_{\max} = 1$.
\end{proof}

\begin{thmlem}\label{lem:classSizeSig}
The maximum number of DNF rules with a given signature $\bK \in \Delta$ is bounded as 
\[
\max_{\bK\in\Delta} \; \lvert\scrC(\bK)\rvert < (2d)^{1/\lambda_1}.
\]
\end{thmlem}
\begin{proof}
The number of conjunctions of degree $p$ is $\binom{d}{p} 2^p$, where the factor of $2^p$ is due to there being two choices of conditions on each of the $p$ selected variables. The number of DNF rules with signature $\bK$ is therefore 
\[
\lvert\scrC(\bK)\rvert = \prod_{p=1}^{p_{\max}} \binom{\binom{d}{p} 2^p}{K_p} < \prod_{p=1}^{p_{\max}} \frac{\left(\binom{d}{p} 2^p\right)^{K_p}}{K_p!}.
\]
Taking logarithms, we obtain 
\begin{align}
\max_{\bK\in\Delta} \; \log &\lvert\scrC(\bK)\rvert <\nonumber\\ 
\max_{\bK} \;\; &\sum_{p=1}^{p_{\max}} K_p \log\left(\binom{d}{p} 2^p\right) - \log(K_p!)\nonumber\\
\text{s.t.} \;\; &\sum_{p=1}^{p_{\max}} K_p (\lambda_0 + p\lambda_1) \leq 1.\label{eqn:classSizeSig}
\end{align}
For simplicity, we drop the nonlinear term $-\log(K_p!) \leq 0$. The right-hand side of \eqref{eqn:classSizeSig} then becomes a maximization of a linear function over a simplex. The maximum value is given by 
\begin{equation}\label{eqn:classSizeSig2}
\max_{p=1,\dots,p_{\max}} \; \frac{\log\left(\binom{d}{p} 2^p\right)}{\lambda_0 + p\lambda_1}
\end{equation}
(attained by setting $K_{p^*} = 1 / (\lambda_0 + p^*\lambda_1)$ for a maximizing value $p^*$ and $K_p = 0$ otherwise). Again for simplicity, we further bound \eqref{eqn:classSizeSig2} from above by dropping $\lambda_0$ from the denominator, resulting in 
\[
\max_{\bK\in\Delta} \; \log \lvert\scrC(\bK)\rvert < \frac{1}{\lambda_1} \max_{p=1,\dots,p_{\max}} \; \frac{1}{p} \log\binom{d}{p} + \log 2
\]
(otherwise \eqref{eqn:classSizeSig2} may require solving a transcendental equation). Since $\log\binom{d}{p}$ increases sublinearly with $p$, the maximum occurs at $p=1$, yielding the desired result.
\end{proof}

By combining \eqref{eqn:classSize}, Lemmas~\ref{lem:simplex} and \ref{lem:classSizeSig}, we obtain the desired bound on the number of DNF rules satisfying the optimality conditions in Lemmas~\ref{lem:maxDeg} and \ref{lem:sigCons}.
\begin{thmlem}\label{lem:classSize}
Under the assumptions of Lemma~\ref{lem:maxDeg}, the number of DNF rules satisfying the necessary conditions of optimality in Lemmas~\ref{lem:maxDeg} and \ref{lem:sigCons} is bounded as 
\[
\lvert\scrC\rvert < 2 (2d)^{1/\lambda_1} \left(\frac{1}{\lambda_1}\right)^{p_{\max}}.
\]
\end{thmlem}

\subsection{Proof of Theorem~\ref{thm:MV}}

\fxnote{The following gives a straightforward bound on the absolute regret. The proof is written out for completeness. Bounding the relative regret (relative to $Q(\cC^*)$, i.e.~with a $1+\epsilon$ factor) is somewhat more involved but doable.}

We prove the theorem in two steps, first relating the empirical estimator in \eqref{eqn:binaryclass} to a problem intermediate between \eqref{eq:MVobj} and \eqref{eqn:binaryclass}, 
\begin{equation}\label{eq:MVobjEmpProb}
\begin{split}
    \cS^* := \argmin_{\cC} \;\; &Q(\cC) := \bar{V}(\cC) + R(\cC)\\
    \text{subject to} \;\; &\sum_{i\in\cI} \mathds{1}[x_i \in \cC] \geq \alpha m, 
\end{split}
\end{equation}
and then relating this intermediate problem \eqref{eq:MVobjEmpProb} to \eqref{eq:MVobj}. Problem \eqref{eq:MVobjEmpProb} has the same regularized volume objective as \eqref{eq:MVobj} but with the empirical probability constraint of \eqref{eqn:binaryclass}.

For the first step, let $\hat{V}(\cC)$ denote the empirical volume in \eqref{eqn:binaryclass} (i.e.~the first term in the objective function). As noted in Section~\ref{sec:MVruleSet}, $\hat{V}(\cC)$ is a scaled binomial random variable with $n$ trials and mean $\bar{V}(\cC)$. Hoeffding's inequality thus provides the following tail bound: 
\[
    \Pr\bigl(\bigl\lvert \hat{V}(\cC) - \bar{V}(\cC)\bigr\rvert > \epsilon_n \bigr) \leq 2 e^{-2n\epsilon_n^2}.
\]
Defining $\hat{Q}(\cC) = \hat{V}(\cC) + R(\cC)$ and recalling that $Q(\cC) = \bar{V}(\cC) + R(\cC)$, the same bound holds for the difference $\hat{Q}(\cC) - Q(\cC)$. Taking the union bound over the hypothesis class $\scrC$ yields 
\begin{equation}\label{eqn:tailBound}
    \Pr\bigl(\exists \cC \in \scrC : \bigl\lvert \hat{Q}(\cC) - Q(\cC)\bigr\rvert > \epsilon_n \bigr) \leq 2 \lvert\scrC\rvert e^{-2n\epsilon_n^2}.
\end{equation}

Assuming that the event in \eqref{eqn:tailBound} is not true, we obtain the following sequence of bounds, where the second inequality is due to the optimality of $\hat{\cS}$ in \eqref{eqn:binaryclass}:
\begin{align}
    Q(\hat{\cS}) &\leq \hat{Q}(\hat{\cS}) + \epsilon_n
    \leq \hat{Q}(\cS^*) + \epsilon_n
    \leq Q(\cS^*) + 2\epsilon_n.\label{eqn:pfMVstep1}
\end{align}
For this to hold with probability at least $1-\delta$, we set $\delta$ equal to the right-hand side of \eqref{eqn:tailBound} to obtain
\begin{equation}\label{eqn:epsilon}
\epsilon_n = \sqrt{\frac{\log(2\lvert\scrC\rvert / \delta)}{2n}}.
\end{equation}

For the second step, we observe that the empirical probability $\hat{P}(\cC) = \frac{1}{m} \sum_{i\in\cI} \mathds{1}[x_i \in \cC]$ is also a scaled binomial random variable, this time with $m$ trials and mean $P(\cC)$. We thus have a similar bound as in \eqref{eqn:tailBound}, 
\[
\Pr\bigl(\exists \cC \in \scrC : \bigl\lvert \hat{P}(\cC) - P(\cC)\bigr\rvert > \epsilon_m \bigr) \leq 2 \lvert\scrC\rvert e^{-2m\epsilon_m^2},
\]
and setting the right-hand side equal to $\delta$ yields the same expression for $\epsilon_m$ as in \eqref{eqn:epsilon} with $n$ replaced by $m$. We then use Theorem~3 and Corollary~12 in \citep{scott2006learning} to conclude that with probability at least $1 - \delta$, 
\[
Q(\cS^*) \leq q^*(\alpha + \epsilon_m) \quad \text{and} \quad P(\cS^*) \geq \alpha - \epsilon_m.
\]
Indeed, since $\hat{S} \in \scrC$ and satisfies the constraint $\hat{P}(\hat{\cS}) \geq \alpha$ as well, the above may be changed to 
\begin{equation}\label{eqn:pfMVstep2}
Q(\cS^*) \leq q^*(\alpha + \epsilon_m) \quad \text{and} \quad P(\hat{\cS}) \geq \alpha - \epsilon_m.
\end{equation}
Combining \eqref{eqn:pfMVstep1} and \eqref{eqn:pfMVstep2} gives \begin{align*}
    Q(\hat{\cS}) &\leq q^*(\alpha + \epsilon_m) + 2\epsilon_n \quad \text{and} \quad 
    P(\hat{\cS}) \geq \alpha - \epsilon_m
\end{align*}
with probability at least $1 - 2\delta$.

Lastly, we use Lemma~\ref{lem:classSize} to bound $\epsilon_n$ from above by 
\[
\sqrt{\frac{\lambda_1^{-1} \log(2d) + p_{\max} \log\lambda_1^{-1} + \log(4/\delta)}{2n}}
\]
and similarly for $\epsilon_m$.

\section{Generalization of the product estimator}
\label{sec:generalizationtheory}
Below, we give a Theorem bounding the expected error of the two-stage estimate $\hcO = \hcS \cap \hcB$ as a function of the error of the base estimators $\hcS, \tcB$. This justifies the two-stage nature of our algorithm and motivates selecting hyperparameters for overlap rules $\hcB$ based on the error with respect to the base estimator $\tcB$.
Before we state the result, we give a Lemma bounding the error of an estimator of a product of functions in terms of estimators of the respective terms in the product. 

Consider the task of predicting the binary deterministic label $g(X) = g_1(X)g_2(X)$ by approximating the product of estimators $f_1, f_2$ of $g_1, g_2$. Now, let $R_g(f)$ denote the expected zero-one loss of $f$ with respect to $g$ over $p$, 
$$
R_g(f) = \E_{X\sim p}[\mathds{1}[f(x) \neq g(x)]]~.
$$
\begin{thmlem}\label{lem:prodclassifier}
For $f_1$ and $f_2$ such that $R_{g_1}(f_1) \leq A \leq \min\{p(f_2(X)=1), p(g_2(X)=1)\}$, $R_{g_2}(f_2) \leq B \leq \min\{p(f_1(X)=1), p(g_1(X)=1)\}$ and $\max\{A + B, C\} \leq 1/2$, let $f(X)$ approximate $f_1(X)f_2(X)$ and assume that $R_{f_1f_2}(f) \leq C$. Then, 
$$
R_g(f) \leq A + B + C
$$
\end{thmlem}
\begin{proof}
For convenience, let $f_1 = f_1(X), g_1 = g_1(X)$, et cetera, and let $\gamma = p(g(X)=1)$.
\begin{align*}
& R_g(f_1f_2) = p(f_1f_2 \neq g_1g_2) \\
& = p(f_1=f_2=1 \land (g_1=0 \lor g_2=0)) \\
& + p((f_1=0 \lor f_2=0) \land g_1=g_2=1) \\
& \leq  p(f_1=f_2=1 \land g_1=0) + p(f_1=f_2=1 \land g_2=0) \\
& + p(g_1=g_2=1 \land f_1=0) + p(g_1=g_2=1 \land f_2=0) \\
& \leq \min\{p(h_2=1), p(f_1=1 \land g_1=0)\} \\
& + \min\{p(f_1=1), p(f_2=1 \land g_2=0)\} \\
& + \min\{p(g_2=1), p(g_1=1 \land f_1=0)\} \\
& + \min\{p(g_1=1), p(g_2=1 \land f_2=0)\} \\
& \leq A + B
\end{align*}
In the first inequality, we use the standard Frechet inequalities. In the second and third, we use the assumptions in the statement. 
Alternatively, we could arrive at the same result by assuming that $h_2$ and $(f_1, h_1)$ as well as $h_1$ and $(f_2, h_2)$ are independent and decomposing the joint distributions. This could be guaranteed by sample splitting. We could then remove the assumption that the marginal probability of the label is larger than the error. 
In either case, 
\begin{align*}
R_g(f) & = p(f = f_1f_2 \land f_1f_2 \neq g) \\
& + p(f \neq f_1f_2 \land f_1f_2 = g) \\
& \leq \min\{p(f = f_1f_2), p(f_1f_2 \neq g)\} \\
& + \min\{p(f \neq f_1f_2), p(f_1f_2 = g)\} \\
& = p(f_1f_2 \neq g) + p(f \neq f_1f_2) \\
& \leq A + B + C~.
\end{align*}
\end{proof}

We now state our result. First, we view membership in $\hcO = \hcS \cap \hcB$ as given by an instance of the hypothesis class $\cF = \{f(x) \coloneqq \mathds{1}[x \in \hcS]h(x); h\in \cH\}$, for some function family $\cH$. Then, let $R_g(f) = \E_{X\sim p}[\mathds{1}[f(x) \neq g(x)]]$ denote the expected risk of $f$ with respect to $g$ over $p$, and $\hR_g(f) = \frac{1}{m}\sum_{i=1}^m \mathds{1}[f(x_i) \neq g(x_i)]$ the empirical risk. 

\begin{thmthm}\label{thm:generalization_app}%
Given are classifiers $\hat{s}, \tilde{b}$ of support membership $s$ and propensity boundedness $b$, with overlap defined as $o(x) = s(x)b(x)$, such that for all $n > N$ it holds for $A_n, C_n \in \tcO(1/\sqrt{n})$ with $\max\{A_n, C_n\} \leq 1/4$ that $R_s(\hat{s}) \leq A_n, R_b(\tilde{b}) \leq C_n$. Then, for any function $\hat{o}\in \cH$ approximating $\hat{s}\cdot \tilde{b}$, with probability larger than $1-\delta$, 
$$
R_o(\ho) \leq \hat{R}_{\hs \cdot \tilde{b}}(\ho) + \frac{D_{\cF, \delta, n}}{\sqrt{n}} + \tcO\left(\frac{1}{\sqrt{n}} \right)~,%
$$%
\end{thmthm}%
where $D_{\cF, \delta, n} = \sqrt{8d(\log\frac{2m}{d}+1)+8\log\frac{4}{\delta}}$, with $d$ the VC-dimension of $\cF$ and $\tilde{O}$ hides logarithmic factors.
\begin{proof}
From Lemma~\ref{lem:prodclassifier} and assumptions, we have that 
$$
R_{o}(\hat{o}) \leq R_{\hat{s}\cdot\tilde{b}}(\hat{o}) + R_{s}(\hat{s}) + R_{b}(\tilde{b}) \leq  R_{\hat{s}\cdot\tilde{b}}(\hat{o}) +  \tcO\left(\frac{1}{\sqrt{n}} \right)~.
$$
By applying standard VC-theory w.r.t. $\cF$, we have our result.
\end{proof}

Theorem~\ref{thm:generalization_app} bounds the generalization error of (e.g., Boolean rule) approximations of $\sqrt{n}$-consistent base estimators. It may be generalized to other rates, but convergence at \emph{some} rate is necessary for consistency of the final estimator. Critically, the bias incurred by the approximation is observable and may be traded off for interpretability.


\bibliography{AISTATS_ArXiv.bbl}

\begin{thebibliography}{}

\bibitem[Angelino et~al., 2017]{angelino2017learning}
Angelino, E., Larus-Stone, N., Alabi, D., Seltzer, M., and Rudin, C. (2017).
\newblock Learning certifiably optimal rule lists.
\newblock In {\em Proceedings of the 23rd ACM SIGKDD International Conference
  on Knowledge Discovery and Data Mining (KDD)}, pages 35--44.

\bibitem[Ben-David et~al., 2010]{ben2010theory}
Ben-David, S., Blitzer, J., Crammer, K., Kulesza, A., Pereira, F., and Vaughan,
  J.~W. (2010).
\newblock A theory of learning from different domains.
\newblock {\em Machine learning}, 79(1-2):151--175.

\bibitem[Benavidez and Frakt, 2019]{guidelines}
Benavidez, G. and Frakt, A.~B. (2019).
\newblock Fixing clinical practice guidelines.
\newblock {Health Affairs Blog, August 5}, Retrieved from:
  \url{https://www.healthaffairs.org/do/10.1377/hblog20190730.874541/full/}.

\bibitem[Brat et~al., 2018]{brat2018postsurgical}
Brat, G.~A., Agniel, D., Beam, A., Yorkgitis, B., Bicket, M., Homer, M., Fox,
  K.~P., Knecht, D.~B., McMahill-Walraven, C.~N., Palmer, N., et~al. (2018).
\newblock Postsurgical prescriptions for opioid naive patients and association
  with overdose and misuse: retrospective cohort study.
\newblock {\em Bmj}, 360:j5790.

\bibitem[Crump et~al., 2009]{crump2009dealing}
Crump, R.~K., Hotz, V.~J., Imbens, G.~W., and Mitnik, O.~A. (2009).
\newblock Dealing with limited overlap in estimation of average treatment
  effects.
\newblock {\em Biometrika}, 96(1):187--199.

\bibitem[D'Amour et~al., 2017]{d2017overlap}
D'Amour, A., Ding, P., Feller, A., Lei, L., and Sekhon, J. (2017).
\newblock Overlap in observational studies with high-dimensional covariates.
\newblock {\em arXiv preprint arXiv:1711.02582}.

\bibitem[Dash et~al., 2018]{dash2018boolean}
Dash, S., Gunluk, O., and Wei, D. (2018).
\newblock Boolean decision rules via column generation.
\newblock In Bengio, S., Wallach, H., Larochelle, H., Grauman, K.,
  Cesa-Bianchi, N., and Garnett, R., editors, {\em Advances in Neural
  Information Processing Systems 31}, pages 4660--4670. Curran Associates, Inc.

\bibitem[Dwork et~al., 2012]{dwork2012fairness}
Dwork, C., Hardt, M., Pitassi, T., Reingold, O., and Zemel, R. (2012).
\newblock Fairness through awareness.
\newblock In {\em Proceedings of the 3rd innovations in theoretical computer
  science conference}, pages 214--226. ACM.

\bibitem[Fogarty et~al., 2016]{fogarty2016discrete}
Fogarty, C.~B., Mikkelsen, M.~E., Gaieski, D.~F., and Small, D.~S. (2016).
\newblock Discrete optimization for interpretable study populations and
  randomization inference in an observational study of severe sepsis mortality.
\newblock {\em Journal of the American Statistical Association},
  111(514):447--458.

\bibitem[Freitas, 2014]{freitas2014comprehensible}
Freitas, A.~A. (2014).
\newblock Comprehensible classification models: a position paper.
\newblock {\em ACM SIGKDD explorations newsletter}, 15(1):1--10.

\bibitem[Fujimoto et~al., 2019]{fujimoto19a}
Fujimoto, S., Meger, D., and Precup, D. (2019).
\newblock Off-policy deep reinforcement learning without exploration.
\newblock In {\em Proceedings of the 36th International Conference on Machine
  Learning}, volume~97, pages 2052--2062. PMLR.

\bibitem[Goh and Rudin, 2015]{goh2015cascaded}
Goh, S.~T. and Rudin, C. (2015).
\newblock Cascaded high dimensional histograms: A generative approach to
  density estimation.
\newblock {\em arXiv preprint arXiv:1510.06779}.

\bibitem[Gupta et~al., 2011]{Gupta2011}
Gupta, K., Hooton, T.~M., Naber, K.~G., Wullt, B., Colgan, R., Miller, L.~G.,
  Moran, G.~J., Nicolle, L.~E., Raz, R., Schaeffer, A.~J., and Soper, D.~E.
  (2011).
\newblock {International clinical practice guidelines for the treatment of
  acute uncomplicated cystitis and pyelonephritis in women: A 2010 update by
  the Infectious Diseases Society of America and the European Society for
  Microbiology and Infectious Diseases.}
\newblock {\em Clinical Infectious Diseases}, 52(5):e103--20.

\bibitem[Herrera et~al., 2011]{herrera2011overview}
Herrera, F., Carmona, C.~J., Gonz{\'a}lez, P., and Del~Jesus, M.~J. (2011).
\newblock An overview on subgroup discovery: foundations and applications.
\newblock {\em Knowledge and information systems}, 29(3):495--525.

\bibitem[Iacus et~al., 2012]{iacus2012causal}
Iacus, S.~M., King, G., and Porro, G. (2012).
\newblock Causal inference without balance checking: Coarsened exact matching.
\newblock {\em Political analysis}, 20(1):1--24.

\bibitem[Johansson et~al., 2019]{johansson2019support}
Johansson, F., Sontag, D., and Ranganath, R. (2019).
\newblock Support and invertibility in domain-invariant representations.
\newblock In {\em The 22nd International Conference on Artificial Intelligence
  and Statistics}, pages 527--536.

\bibitem[Kahn, 1955]{Kahn1955}
Kahn, H. (1955).
\newblock {Use of Different Monte Carlo Sampling Techniques}.
\newblock Technical report, RAND Corporation, Santa Monica, California.

\bibitem[Kallus, 2016]{kallus2016generalized}
Kallus, N. (2016).
\newblock Generalized optimal matching methods for causal inference.
\newblock {\em arXiv preprint arXiv:1612.08321}.

\bibitem[Kallus and Zhou, 2018]{pmlr-v84-kallus18a}
Kallus, N. and Zhou, A. (2018).
\newblock Policy evaluation and optimization with continuous treatments.
\newblock {\em Proceedings of the Twenty-First International Conference on
  Artificial Intelligence and Statistics}, 84:1243--1251.

\bibitem[Lakkaraju et~al., 2016]{lakkaraju2016interpretable}
Lakkaraju, H., Bach, S.~H., and Leskovec, J. (2016).
\newblock Interpretable decision sets: A joint framework for description and
  prediction.
\newblock In {\em Proceedings of the 22nd ACM SIGKDD international conference
  on knowledge discovery and data mining}, pages 1675--1684. ACM.

\bibitem[LaLonde, 1986]{lalonde1986evaluating}
LaLonde, R.~J. (1986).
\newblock Evaluating the econometric evaluations of training programs with
  experimental data.
\newblock {\em The American economic review}, pages 604--620.

\bibitem[Li et~al., 2018]{li2018balancing}
Li, F., Morgan, K.~L., and Zaslavsky, A.~M. (2018).
\newblock Balancing covariates via propensity score weighting.
\newblock {\em Journal of the American Statistical Association},
  113(521):390--400.

\bibitem[{National Cancer Institute}, 2012]{myelomatrial2}
{National Cancer Institute} (2012).
\newblock Bortezomib in treating patients with newly diagnosed multiple
  myeloma. {ClinicalTrials.gov Identifier NCT00075881}.
\newblock Retrieved from:
  \url{https://clinicaltrials.gov/ct2/show/NCT00075881}.

\bibitem[Pearl, 2009]{pearl2009causality}
Pearl, J. (2009).
\newblock {\em Causality}.
\newblock Cambridge university press.

\bibitem[Precup et~al., 2000]{Precup2000}
Precup, D., Sutton, R.~S., and Singh, S.~P. (2000).
\newblock {Eligibility Traces for Off-Policy Policy Evaluation}.
\newblock In {\em Proceedings of the Seventeenth International Conference on
  Machine Learning (ICML)}, pages 759--766.

\bibitem[Ram and Gray, 2011]{ram2011density}
Ram, P. and Gray, A.~G. (2011).
\newblock Density estimation trees.
\newblock In {\em Proceedings of the 17th ACM SIGKDD international conference
  on Knowledge discovery and data mining}, pages 627--635. ACM.

\bibitem[Rivest, 1987]{rivest1987learning}
Rivest, R.~L. (1987).
\newblock Learning decision lists.
\newblock {\em Machine learning}, 2(3):229--246.

\bibitem[Rosenbaum, 1989]{rosenbaum1989optimal}
Rosenbaum, P.~R. (1989).
\newblock Optimal matching for observational studies.
\newblock {\em Journal of the American Statistical Association},
  84(408):1024--1032.

\bibitem[Rosenbaum, 2010]{rosenbaum2010design}
Rosenbaum, P.~R. (2010).
\newblock {\em Design of observational studies}, volume~10.
\newblock Springer.

\bibitem[Rosenbaum and Rubin, 1983]{rosenbaum1983central}
Rosenbaum, P.~R. and Rubin, D.~B. (1983).
\newblock The central role of the propensity score in observational studies for
  causal effects.
\newblock {\em Biometrika}, 70(1):41--55.

\bibitem[Sanchez et~al., 2016]{Sanchez2016}
Sanchez, G.~V., Babiker, A., Master, R.~N., Luu, T., Mathur, A., and Bordon, J.
  (2016).
\newblock {Antibiotic Resistance among Urinary Isolates from Female Outpatients
  in the United States in 2003 and 2012.}
\newblock {\em Antimicrobial Agents and Chemotherapy}, 60(5):2680--2683.

\bibitem[Sch{\"o}lkopf et~al., 2001]{scholkopf2001estimating}
Sch{\"o}lkopf, B., Platt, J.~C., Shawe-Taylor, J., Smola, A.~J., and
  Williamson, R.~C. (2001).
\newblock Estimating the support of a high-dimensional distribution.
\newblock {\em Neural computation}, 13(7):1443--1471.

\bibitem[Scott and Nowak, 2006]{scott2006learning}
Scott, C.~D. and Nowak, R.~D. (2006).
\newblock Learning minimum volume sets.
\newblock {\em Journal of Machine Learning Research}, 7(Apr):665--704.

\bibitem[Shapiro et~al., 2013]{Shapiro2013}
Shapiro, D.~J., Hicks, L.~A., Pavia, A.~T., and Hersh, A.~L. (2013).
\newblock {Antibiotic prescribing for adults in ambulatory care in the USA,
  2007–09}.
\newblock {\em Journal of Antimicrobial Chemotherapy}, 69(1):234--240.

\bibitem[Smith and Todd, 2005]{smith2005does}
Smith, J.~A. and Todd, P.~E. (2005).
\newblock Does matching overcome lalonde's critique of nonexperimental
  estimators?
\newblock {\em Journal of econometrics}, 125(1-2):305--353.

\bibitem[Su et~al., 2016]{su2016learning}
Su, G., Wei, D., Varshney, K.~R., and Malioutov, D.~M. (2016).
\newblock Learning sparse two-level {Boolean} rules.
\newblock In {\em Proc. IEEE Int. Workshop Mach. Learn. Signal Process.
  (MLSP)}, pages 1--6.

\bibitem[Sutton and Barto, 2017]{Sutton2017}
Sutton, R.~S. and Barto, A.~G. (2017).
\newblock {\em {Reinforcement Learning: An Introduction}}.
\newblock MIT Press, 2nd edition.

\bibitem[Visconti and Zubizarreta, 2018]{visconti2018handling}
Visconti, G. and Zubizarreta, J.~R. (2018).
\newblock Handling limited overlap in observational studies with cardinality
  matching.
\newblock {\em Observational Studies}, 4:217--249.

\bibitem[Wang and Rudin, 2015]{wang2015falling}
Wang, F. and Rudin, C. (2015).
\newblock Falling rule lists.
\newblock In {\em Artificial Intelligence and Statistics}, pages 1013--1022.

\bibitem[Wang et~al., 2017]{wang2017bayesian}
Wang, T., Rudin, C., Doshi-Velez, F., Liu, Y., Klampfl, E., and MacNeille, P.
  (2017).
\newblock A {Bayesian} framework for learning rule sets for interpretable
  classification.
\newblock {\em Journal of Machine Learning Research}, 18(70):1--37.

\bibitem[Wei et~al., 2019]{wei2019generalized}
Wei, D., Dash, S., Gao, T., and Gunluk, O. (2019).
\newblock Generalized linear rule models.
\newblock In {\em Proceedings of the 36th International Conference on Machine
  Learning (ICML)}.

\bibitem[Yang et~al., 2017]{yang2017scalable}
Yang, H., Rudin, C., and Seltzer, M. (2017).
\newblock Scalable {Bayesian} rule lists.
\newblock In {\em Proc. Int. Conf. Mach. Learn. (ICML)}, pages 1013--1022.

\bibitem[Yau and Zhang, 2006]{yau2006simplex}
Yau, S.~T. and Zhang, L. (2006).
\newblock An upper estimate of integral points in real simplices with an
  application to singularity theory.
\newblock {\em Math. Res. Lett.}, 13(6):911--921.

\bibitem[Zhang et~al., 2017]{zhang2017exploring}
Zhang, J., Iyengar, V., Wei, D., Vinzamuri, B., Bastani, H.~S., Macalalad,
  A.~R., Fischer, A.~E., Yuen-Reed, G., Mojsilovic, A., and Varshney, K.~R.
  (2017).
\newblock Exploring the causal relationships between initial opioid
  prescriptions and outcomes.
\newblock In {\em AMIA Workshop on Data Mining for Medical Informatics},
  Washington, DC.

\bibitem[Zubizarreta, 2012]{zubizarreta2012using}
Zubizarreta, J.~R. (2012).
\newblock Using mixed integer programming for matching in an observational
  study of kidney failure after surgery.
\newblock {\em Journal of the American Statistical Association},
  107(500):1360--1371.

\end{thebibliography}

\onecolumn
{\small
\begin{longtable}{lcccc}
    \caption{Population averages for covariates in Opioids in order of difference between the overlapping and non-overlapping set. DMME, MME and Duration are the medians of daily MME, total MME and prescription duration days in each group.}\\%
    \toprule
    & {\bf Total} & {\bf DMME} & {\bf MME} & {\bf Duration} \\
    Total sample & 35106 & 46 & 225 & 5 \\
    Male & 9301 & 50 & 300 & 5 \\
    Female & 25805 & 45 & 225 & 5 \\ \midrule
    {\bf Age groups} \\
    <15 & 847 & 20 & 100 & 5 \\
    15-24 & 3334 & 45 & 200 & 5 \\
    25-34 & 9994 & 45 & 210 & 4 \\
    35-44 & 6820 & 46 & 225 & 5 \\
    45-54 & 6196 & 50 & 250 & 5 \\
    55-64 & 7915 & 50 & 300 & 5 \\
    >=65 & 0 & 0 & 0 & 0 \\\midrule
    {\bf Surgery type} \\
    Auditory & 29 & 18 & 135 & 6 \\
    Cardiovascular & 3633 & 45 & 270 & 5 \\
    Integumentary & 1507 & 48 & 225 & 5 \\
    Mediastinum & 54 & 47 & 300 & 5 \\
    Female genital & 3913 & 48 & 225 & 5 \\
    Hemic & 885 & 50 & 225 & 5 \\
    Respiratory & 665 & 45 & 250 & 5 \\
    Endocrine & 214 & 45 & 200 & 5 \\
    Nervous & 4350 & 60 & 375 & 6 \\
    Urinary & 1476 & 45 & 225 & 5 \\
    Musculoskeletal & 6678 & 60 & 450 & 7 \\
    Maternity & 13553 & 45 & 200 & 4 \\
    Male genital & 585 & 45 & 225 & 5 \\ \midrule
    {\bf Year} \\
    2011 & 7547 & 45 & 225 & 5 \\
    2012 & 10743 & 46 & 225 & 5 \\
    2013 & 9651 & 50 & 225 & 5 \\
    2014 & 7165 & 45 & 225 & 5 \\ \midrule
    {\bf Diagnosis history (until day before surgery)} \\
    Other specified gastritis: without mention of hemorrhage & 491 & 42 & 225 & 5 \\
    Other ascites & 233 & 45 & 225 & 5 \\
    Lumbosacral spondylosis without myelopathy & 1135 & 60 & 400 & 6 \\
    Nausea with vomiting & 1914 & 45 & 225 & 5 \\
    Other respiratory abnormalities & 1935 & 45 & 225 & 5 \\
    Vomiting alone & 765 & 45 & 200 & 5 \\
    Myalgia and myositis: unspecified & 1522 & 50 & 250 & 5 \\
    Attention deficit disorder with hyperactivity & 370 & 45 & 225 & 5 \\
    Attention deficit disorder without mention of hyperactivity & 444 & 45 & 225 & 5 \\
    Depressive disorder: not elsewhere classified & 2221 & 50 & 225 & 5 \\
    Dysthymic disorder & 752 & 50 & 225 & 5 \\
    Tachycardia: unspecified & 631 & 45 & 225 & 5 \\
    Degeneration of cervical intervertebral disc & 904 & 56 & 337 & 6 \\
    Flatulence: eructation: and gas pain & 427 & 45 & 225 & 5 \\
    Generalized anxiety disorder & 833 & 45 & 225 & 5 \\
    Other symptoms referable to back & 368 & 50 & 300 & 5 \\
    Cellulitis and abscess of leg: except foot & 450 & 45 & 225 & 5 \\
    Constipation: unspecified & 1136 & 45 & 225 & 5 \\
    Thoracic or lumbosacral neuritis or radiculitis: unspecified & 1676 & 60 & 326 & 6 \\
    Anxiety state: unspecified & 2205 & 50 & 225 & 5 \\
    Lumbago & 4559 & 50 & 250 & 5 \\
    Abdominal pain: generalized & 1607 & 45 & 225 & 5 \\
    Degeneration of lumbar or lumbosacral intervertebral disc & 1542 & 60 & 388 & 6 \\
    Other and unspecified noninfectious gastroenteritis and colitis & 1254 & 45 & 225 & 5 \\
    Major depressive affective disorder: recurrent episode: moderate & 507 & 45 & 225 & 5 \\
    Asthma: unspecified type: unspecified & 2044 & 45 & 225 & 5 \\
    Arthrodesis status & 178 & 60 & 450 & 7 \\
    Chest pain: unspecified & 4701 & 45 & 225 & 5 \\
    Routine general medical examination at a health care facility & 9529 & 50 & 225 & 5 \\
    Diarrhea & 1714 & 50 & 225 & 5 \\
    Fitting and adjustment of vascular catheter & 318 & 45 & 225 & 5 \\
    Hypopotassemia & 721 & 45 & 225 & 5 \\
    Bariatric surgery status & 302 & 40 & 200 & 5 \\
    Sprain of neck & 816 & 50 & 225 & 5 \\
    Unspecified gastritis and gastroduodenitis: without mention of hemorrhage & 960 & 45 & 225 & 5 \\
    Injury of face and neck & 271 & 46 & 300 & 5 \\
    Backache: unspecified & 2471 & 50 & 225 & 5 \\
    Unspecified septicemia & 222 & 45 & 225 & 5 \\
    Acute pharyngitis & 4219 & 45 & 225 & 5 \\
    Acute bronchitis & 3311 & 46 & 225 & 5 \\
    Abdominal pain: other specified site & 2890 & 45 & 225 & 5 \\
    Atrophic gastritis: without mention of hemorrhage & 537 & 45 & 225 & 5 \\
    Cough & 3946 & 45 & 225 & 5 \\
    Altered mental status & 202 & 45 & 225 & 5 \\
    Cervicalgia & 2758 & 50 & 250 & 5 \\
    Abdominal pain: unspecified site & 6339 & 45 & 225 & 5 \\
    Other chronic pain & 346 & 56 & 300 & 6 \\
    Headache & 3514 & 45 & 225 & 5 \\
    Tobacco use disorder & 1834 & 50 & 225 & 5 \\
    Other screening mammogram & 5722 & 50 & 240 & 5 \\
    Observation and evaluation for other specified suspected conditions & 337 & 45 & 225 & 5 \\
    Unspecified sinusitis (chronic) & 1624 & 46 & 225 & 5 \\
    Rheumatoid arthritis & 353 & 50 & 300 & 5 \\
    Brachial neuritis or radiculitis NOS & 1147 & 50 & 300 & 5 \\
    Loss of weight & 455 & 46 & 225 & 5 \\
    Hypersomnia with sleep apnea: unspecified & 424 & 42 & 225 & 5 \\
    Insomnia: unspecified & 968 & 50 & 225 & 5 \\
    Other malaise and fatigue & 5178 & 46 & 225 & 5 \\
    Other injury of chest wall & 210 & 50 & 300 & 5 \\
    Dehydration & 841 & 45 & 225 & 5 \\
    Acute respiratory failure & 120 & 40 & 225 & 5 \\
 \bottomrule
 \label{tbl:covariates_supp}
\end{longtable}
}

{\small
\begin{longtable}{lcc}
    \caption{Population averages for the 156 features in the UTI cohort.  Mean values and total (for binary features) are given, and there are 64593 subjects in total.} \label{tbl:covariates_uti} \\
    \toprule
& {\bf Mean} & {\bf Total} \\
{\bf Demographics} \\
Age & 55.1 &  \\
Male & 16.53\% & 10685 \\
White & 72.17\% & 46662 \\
Veteran & 4.61\% & 2981 \\ \midrule
{\bf Current Location}\\
Outpatient & 64.89\% & 41957 \\
Emergency Room & 15.69\% & 10142 \\
Inpatient & 17.26\% & 11159 \\
Intensive Care Unit (ICU) & 2.69\% & 1736 \\ \midrule
{\bf Local Resistance Rates (Past 30-90 days, at this location)} \\
Trimethoprim/Sulfamethoxazole & 18.61\% &  \\
Nitrofurantoin & 19.85\% &  \\
Ciprofloxacin & 22.70\% &  \\
Levofloxacin & 24.19\% &  \\ \midrule
{\bf Secondary Site of Infection} \\
Skin / Soft Tissue & 0.20\% & 132 \\
Blood & 1.59\% & 1031 \\
Respiratory Tract & 0.53\% & 341 \\
Nasal or Rectal Swab & 0.19\% & 124 \\ \midrule
{\bf Medical History (Past 90 Days)} \\
Alcohol abuse & 1.66\% & 1074 \\
Deficiency anemia & 2.84\% & 1837 \\
Cardiac arrhythmias & 17.08\% & 11041 \\
Blood loss anemia & 0.49\% & 315 \\
Congestive heart failure & 10.16\% & 6571 \\
Coagulopathy & 3.81\% & 2466 \\
Diabetes, uncomplicated & 14.13\% & 9135 \\
Diabetes, complicated & 5.00\% & 3232 \\
Depression & 11.80\% & 7627 \\
Drug abuse & 1.72\% & 1114 \\
Fluid and electrolyte disorders & 13.84\% & 8946 \\
AIDS/HIV & 0.43\% & 281 \\
Hypertension, uncomplicated & 32.51\% & 21017 \\
Hypertension, complicated & 5.43\% & 3513 \\
Hypothyroidism & 7.86\% & 5085 \\
Liver disease & 4.36\% & 2822 \\
Lymphoma & 1.63\% & 1051 \\
Metastatic cancer & 5.50\% & 3559 \\
Other neurological disorders & 6.68\% & 4319 \\
Obesity & 6.70\% & 4332 \\
Pulmonary circulation disorders & 3.13\% & 2025 \\
Peptic ulcer disease, excluding bleeding & 0.61\% & 393 \\
Peripheral vascular disorders & 5.68\% & 3672 \\
Paralysis & 3.08\% & 1992 \\
Psychoses & 2.42\% & 1563 \\
Chronic pulmonary disease & 11.29\% & 7299 \\
Renal & 8.87\% & 5735 \\
Rheumatoid arthritis / collagen vascular diseases & 3.76\% & 2428 \\
Solid tumor without metastasis & 12.00\% & 7760 \\
Valvular disease & 7.79\% & 5034 \\
Weight loss & 3.59\% & 2319 \\
Preganant & 3.08\% & 1989 \\ \midrule
{\bf Previous Care (Past 90 days)}\\
Inpatient Stay & 18.38\% & 11882 \\
Nursing Home Stay & 1.20\% & 779 \\ \midrule
{\bf Previous Procedures (Past 90 days)} \\
Central Venous Catheder & 5.27\% & 3410 \\
Hemodialysis & 0.66\% & 427 \\
Mechanical Ventilation & 5.74\% & 3714 \\
Parenteral Nutrition & 0.67\% & 434 \\
Surgery & 59.84\% & 38689 \\ \midrule
{\bf Previous Organisms (Past 90 days)} \\
Citrobacter species & 0.42\% & 270 \\
Coagulate negative Staphylococcus species & 1.15\% & 741 \\
Enterobacter aerogenes & 0.15\% & 95 \\
Escherichia coli & 7.82\% & 5057 \\
Enterococcus species & 2.66\% & 1718 \\
Enterobacter cloacae & 0.29\% & 186 \\
Group B Streptococcus & 0.17\% & 109 \\
Klebsiella pneumoniae & 2.02\% & 1307 \\
Morganella species & 0.11\% & 73 \\
Pseudomonas aeruginosa & 0.92\% & 594 \\
Proteus species & 0.69\% & 445 \\
Staph aureus & 1.55\% & 1003 \\
Serratia species & 0.22\% & 145 \\ \midrule
{\bf Previous Resistance, measured by culture (Last 90 Days)} \\
Amoxicillin Clavulanate & 2.34\% & 1511 \\
Amikacin & 0.10\% & 67 \\
Ampicillin & 7.44\% & 4808 \\
Aztreonam & 0.95\% & 616 \\
Ceftazidime & 0.30\% & 197 \\
Cefazolin & 9.22\% & 5962 \\
Chlorampenicol & 0.17\% & 111 \\
Ciprofloxacin & 4.62\% & 2984 \\
Clindamycin & 0.97\% & 624 \\
Ceftriaxone & 1.24\% & 804 \\
Doxycycline & 0.39\% & 249 \\
Ertapenem & 0.14\% & 88 \\
Erythromycin & 3.71\% & 2399 \\
Cefepime & 0.54\% & 351 \\
Cefoxitin & 0.49\% & 319 \\
Gentamicin & 1.65\% & 1066 \\
Gentamicin (Synergistic) & 0.47\% & 307 \\
Imipenem & 0.47\% & 303 \\
Levofloxacin & 5.32\% & 3439 \\
Linezolid & 0.09\% & 58 \\
Meropenem & 0.13\% & 85 \\
Moxifloxacin & 0.86\% & 556 \\
Nalidixic Acid & 0.09\% & 60 \\
Nitrofurantoin & 4.06\% & 2628 \\
Oxacillin & 1.79\% & 1158 \\
Penicillin & 2.41\% & 1559 \\
Piperacillin & 0.62\% & 402 \\
Polymyxin B & 1.22\% & 790 \\
Rifampin & 0.80\% & 518 \\
Ampicillin Sulbactam & 1.63\% & 1056 \\
Streptomycin (Synergistic) & 0.23\% & 150 \\
Trimethoprim Sulfamethoxazole & 3.10\% & 2006 \\
Tetracycline & 5.33\% & 3443 \\
Ticarcillin & 0.24\% & 153 \\
Tobramycin & 0.31\% & 203 \\
Piperacillin Tazobactam & 0.53\% & 341 \\
Vancomycin & 0.92\% & 598 \\ \midrule
{\bf Previous Antibiotic Prescription (Last 90 Days)}\\
Amikacin & 0.09\% & 60 \\
Amoxicillin & 2.47\% & 1596 \\
Amoxicillin/Clavulanate & 2.15\% & 1388 \\
Amphotericin B & 0.16\% & 102 \\
Ampicillin/Sulbactam & 0.34\% & 217 \\
Azithromycin & 2.86\% & 1847 \\
Aztreonam & 0.25\% & 159 \\
Cefadroxil & 0.15\% & 96 \\
Cefazolin & 4.87\% & 3150 \\
Cefepime & 2.30\% & 1489 \\
Cefixime & 0.26\% & 166 \\
Cefotetan & 0.18\% & 114 \\
Cefoxitin & 0.25\% & 161 \\
Cefpodoxime & 0.88\% & 570 \\
Ceftazidime & 0.73\% & 475 \\
Ceftriaxone & 2.75\% & 1775 \\
Cefuroxime & 0.24\% & 156 \\
Cephalexin & 2.31\% & 1496 \\
Ciprofloxacin & 11.09\% & 7170 \\
Clarithromycin & 0.35\% & 226 \\
Clindamycin & 1.84\% & 1187 \\
Daptomycin & 0.10\% & 63 \\
Dicloxacillin & 0.19\% & 126 \\
Doxycycline & 1.73\% & 1119 \\
Ertapenem & 0.22\% & 140 \\
Erythromycin & 0.39\% & 249 \\
Fluconazole & 3.56\% & 2301 \\
Fosfomycin & 0.36\% & 232 \\
Gentamicin & 0.94\% & 607 \\
Imipenem & 0.33\% & 216 \\
Levofloxacin & 5.94\% & 3838 \\
Linezolid & 0.73\% & 470 \\
Meropenem & 0.40\% & 256 \\
Metronidazole & 4.49\% & 2906 \\
Micafungin & 0.24\% & 154 \\
Minocycline & 0.20\% & 129 \\
Moxifloxacin & 0.27\% & 174 \\
Nafcillin & 0.24\% & 157 \\
Nitrofurantoin & 2.73\% & 1767 \\
Norfloxacin & 4.25\% & 2749 \\
Penicillin & 0.31\% & 199 \\
Piperacillin & 0.41\% & 268 \\
Piperacillin/Tazobactam & 0.23\% & 148 \\
Polymyxin B & 0.52\% & 333 \\
Posaconazole & 0.18\% & 118 \\
Tetracycline Metronidazole & 0.09\% & 59 \\
Trimethoprim & 0.12\% & 79 \\
Trimethoprim/Sulfamethoxazole & 3.96\% & 2558 \\
Vancomycin & 8.80\% & 5690 \\
Vancomycin Gentamicin & 3.35\% & 2165 \\
 \bottomrule
\end{longtable}
}
\twocolumn

\end{document}